%% file: iclr2027_conference.tex
\documentclass{article} 
\usepackage{iclr2027_conference,times}

\input{math_commands.tex}

\usepackage{hyperref}
\usepackage{url}
\usepackage{amsthm}
\usepackage{graphicx}
\usepackage{subcaption}
\usepackage{booktabs}
\usepackage{array}
\usepackage{multirow}
\usepackage{wrapfig}
\usepackage{color}
\usepackage{algorithm}
\usepackage{algpseudocode}
\usepackage{amsmath}

\newtheorem{hypothesis}{Hypothesis}
\newtheorem{assumption}{Assumption}
\newtheorem{proposition}{Proposition}

\title{Learning Normal Diffusion Dynamics for \\ Backdoor Defense in Text-to-Image Models}

\author{
\begin{tabular}[t]{@{}l@{}}
Junjian Li\textsuperscript{1}\thanks{Equal contribution.}, Xiaolong Liu\textsuperscript{2}\footnotemark[1], Peng Sun\textsuperscript{2}\thanks{Corresponding authors.}, Liantao Wu\textsuperscript{3}, Linghan Chen\textsuperscript{4}, \\ 
\textbf{Yudong Gao}\textsuperscript{5}, \textbf{Honglong Chen}\textsuperscript{6}\footnotemark[2] \\
\normalfont \textsuperscript{1}Geely, \textsuperscript{2}Hunan University, \textsuperscript{3}East China Normal University,\\
\normalfont \textsuperscript{4}University of Adelaide, \textsuperscript{5}The Hong Kong University of Science and Technology, \\ 
\normalfont \textsuperscript{6}China University of Petroleum (East China) \\
\normalfont \texttt{psun@hnu.edu.cn}, \texttt{chenhl@upc.edu.cn}
\end{tabular}
}

\iclrfinalcopy 

\begin{document}

\maketitle

\begin{abstract}
Backdoor attacks pose a serious threat to the secure deployment of text-to-image (T2I) diffusion models. Existing defenses typically detect backdoors from specific abnormal patterns in internal representations, which may limit their generalizability with the emergence of increasingly diverse attack mechanisms. In this paper, we study backdoor defense of T2I diffusion models from a transition-dynamics perspective. We observe that benign diffusion trajectories exhibit structured and timestep-dependent transition patterns from cross-attention, latent and noise spaces, whereas backdoor attacks tend to induce deviations from such normal evolution. Motivated by these observations, we propose \textbf{Normal Diffusion Dynamics Learning (\textit{NDDL})}, 
a novel backdoor defense framework that learns the normal transition dynamics of diffusion trajectories utilizing only benign samples. \textit{NDDL} constructs compact multi-space trajectory representations and trains a timestep-conditioned dynamics model to predict the diffusion evolution. In the inference phase, deviations between the observed and predicted transitions are exploited to quantify dynamics inconsistency for backdoor detection. \textit{NDDL} further enables trigger localization without any prior knowledge of the embedded backdoor by performing substitution with low-semantic words. Extensive experiments for diverse backdoor attacks demonstrate the effectiveness and generalizability of our proposed \textit{NDDL}. 
\end{abstract}

\section{INTRODUCTION}
Text-to-image (T2I) diffusion models have achieved remarkable success in high-quality image synthesis, facilitating widespread real-world applications~\citep{t2i2,t2i4,t2i5,t2i6}. However, the increasing prevalence of publicly available models introduces substantial security risks~\citep{risk1,risk2,risk3,risk4,new3,new4}. In particular, backdoor attacks can embed hidden behaviors in T2I diffusion models~\citep{backdoor1,backdoor2}. Thus, a backdoored model behaves normally on benign prompts while generating the attacker-specified contents once the trigger is present~\citep{badt2i,eviledit,MasqLoRA,rickrolling}. Since the downstream users typically have no prior knowledge of the attack mechanism, the reliable backdoor defense is essential for the secure deployment of T2I diffusion models~\citep{zhang2025adversarial,zhang2026trustworthy}.   

Existing defense strategies typically exploit abnormal behaviors induced by backdoors, including distinct patterns in attention, noise prediction, neuron activations ~\citep{t2ishield,terd,NaviT2I}. While these methods indicate that backdoor attacks can leave detectable traces, the resulting detection criteria are usually associated with particular representations or abnormal phenomena. This presents a fundamental challenge for general backdoor defense. Since different types of attacks may rely on different mechanisms or objectives, the internal activation patterns of backdoored models vary across representation spaces and denoising stages~\citep{NaviT2I,stediff}. Consequently, the defenses based on attack-specific characteristics limit the generalizability with the emergence of increasingly diverse backdoor attacks.


We therefore reconsider the defense problem from a different perspective: \textbf{Can we shift the focus from how backdoor behaviors appear abnormal to how benign diffusion normally behaves?} This perspective is natural for T2I diffusion models, whose generation process is based on a sequence of timestep-dependent denoising transitions~\citep{t2i1,li2024alleviating,t2i2,t2i3,transition,song2020score}. 

\begin{wrapfigure}{r}{0.5\textwidth}
    \centering
    \begin{subfigure}{0.24\columnwidth}
        \centering
        \includegraphics[width=\textwidth]{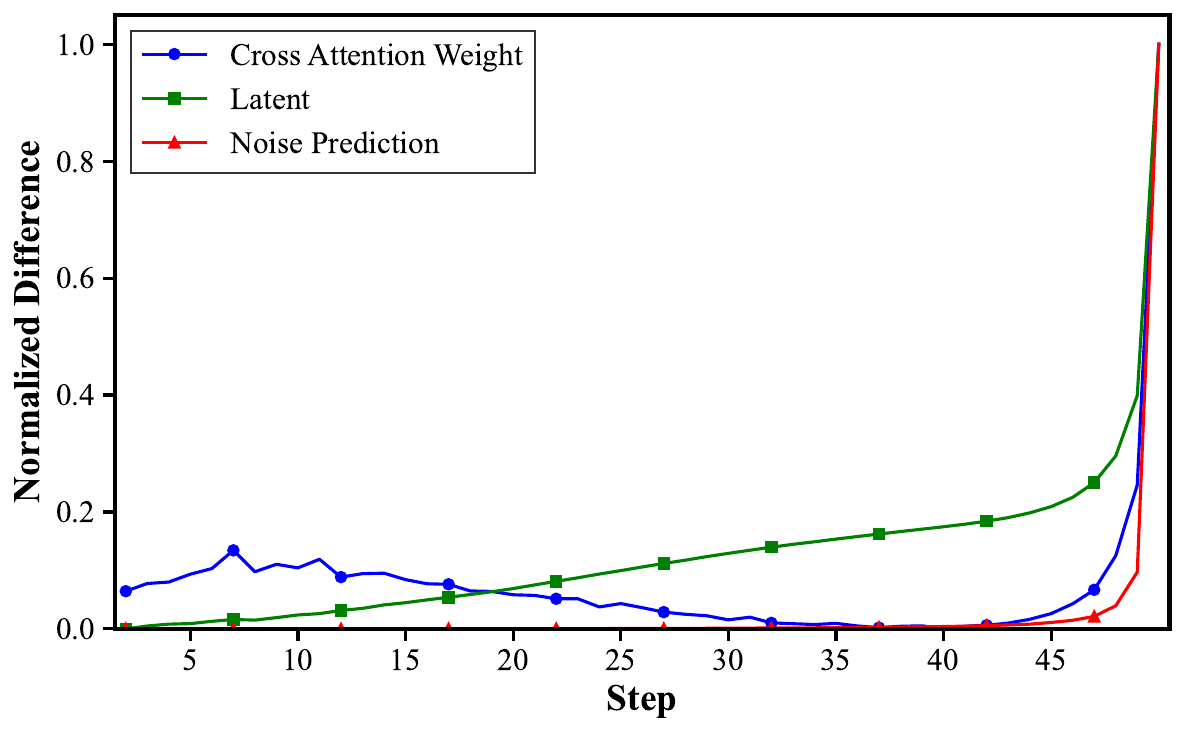}
        \caption{Averaged curves}
        \label{ob1:a}
    \end{subfigure}
    \hfill 
    \begin{subfigure}{0.24\columnwidth}
        \centering
        \includegraphics[width=\textwidth]{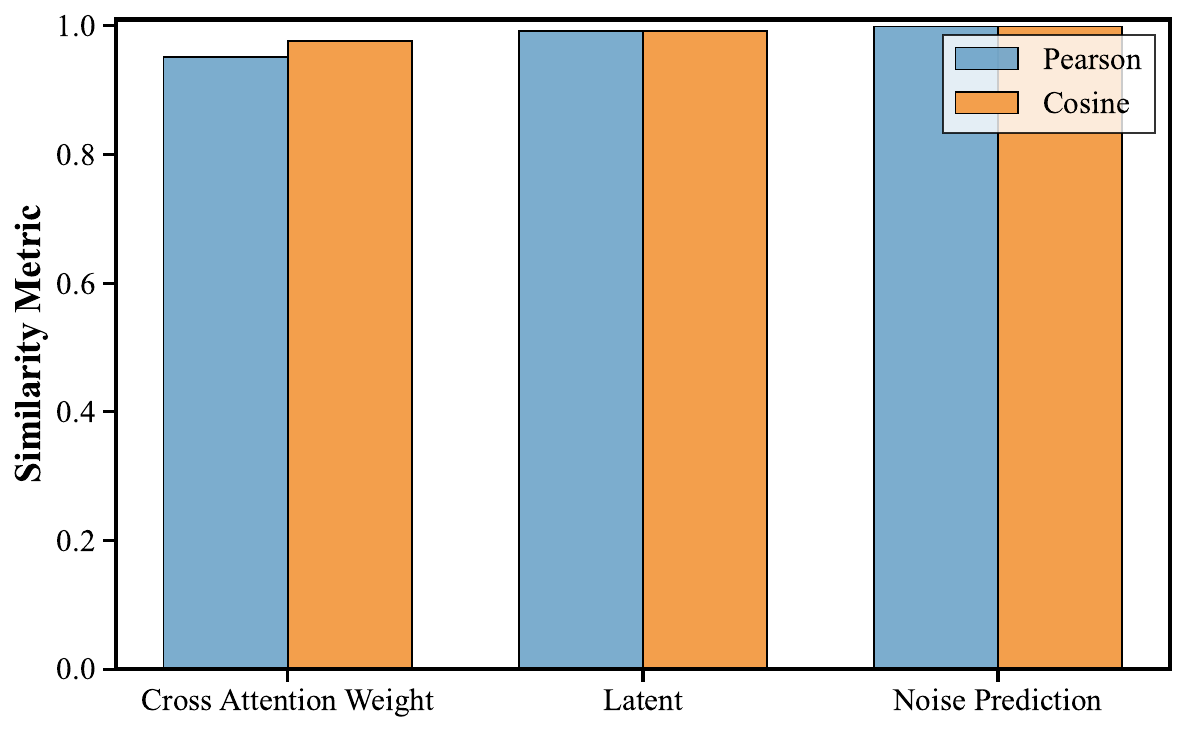}
        \caption{Similarity results}
        \label{ob1:b}
    \end{subfigure}
    \caption{Timestep-dependent evolution patterns of transition differences for the three representations on Stable Diffusion v1.5. We measure the transition differences for 1000 benign prompts. Figure~\ref{ob1:a} shows the averaged curves of the three representations. Figure~\ref{ob1:b} presents the Pearson correlation and cosine similarity between each individual transition trajectory and the corresponding mean results. More details and observation results are available in Appendix~\ref{ob1_sub_appendix}.}
    \label{ob1}
\end{wrapfigure}

Instead of inducing an obvious anomaly at the particular timestep, backdoor may perturb the denoising transitions and gradually deviate the generation from the normal evolution. To this end, we analyze the model diffusion trajectories from cross-attention, latent and noise spaces. Our empirical study reveals two key observations: (1) \textbf{Observation I:} As shown in Figure~\ref{ob1}, \textbf{benign trajectories present structured and timestep-dependent transition patterns}; (2) \textbf{Observation II:} As shown in Figure~\ref{ob2}, \textbf{backdoors consistently introduce deviations to the normal evolution across different attacks}. These observations suggest that benign diffusion exhibits learnable transition rule, while backdoor tends to introduce perturbations. Motivated by the above analysis, we propose \textbf{Normal} \textbf{Diffusion} \textbf{Dynamics} \textbf{Learning} (\textit{\textbf{\textit{NDDL}}}), a novel backdoor defense framework for learning the normal transition dynamics of diffusion trajectories utilizing only benign samples. \textit{NDDL} first constructs compact trajectory representations from cross-attention, latent and noise spaces. Then, a dynamics model is trained on benign trajectories to predict the next-step representation. During inference, the deviations between the observed and predicted transitions are utilized to quantify dynamics inconsistency, with large deviations indicating potential backdoor attack. Furthermore, \textit{NDDL} can localize the potential trigger tokens by performing substitution with low-semantic words, enabling trigger identification without requiring prior knowledge of the embedded backdoor.
\vspace{-0.8em}
\begin{figure}[htbp]
    \centering
    \begin{subfigure}[b]{0.32\textwidth}
        \centering
        \includegraphics[width=\textwidth]{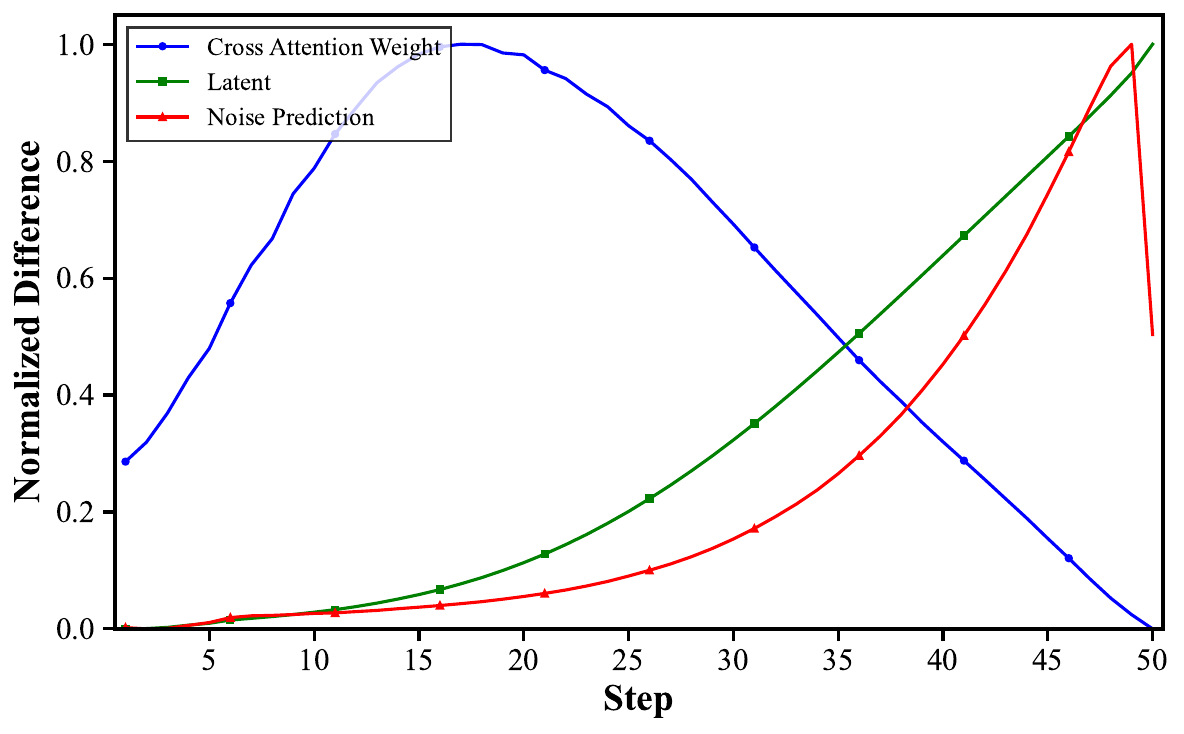} 
        \caption{BadT2I (Pixel)}
        \label{ob2:a}
    \end{subfigure}
    \hfill 
    \begin{subfigure}[b]{0.32\textwidth}
        \centering
        \includegraphics[width=\textwidth]{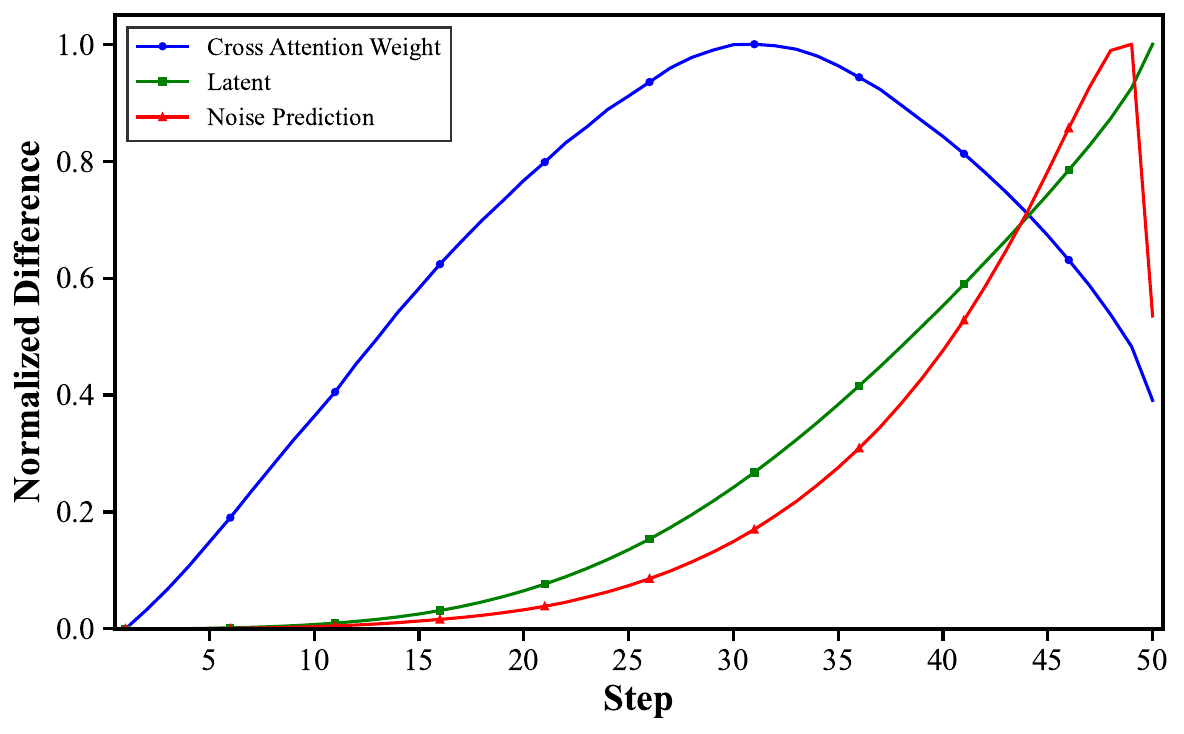}
        \caption{EvilEdit}
        \label{ob2:b}
    \end{subfigure}
    \hfill
    \begin{subfigure}[b]{0.32\textwidth}
        \centering
        \includegraphics[width=\textwidth]{fig/observation/ob2/MasqLoRA-main.pdf}
        \caption{MasqLoRA}
        \label{ob2:c}
    \end{subfigure}
    
    \caption{Discrepancies in diffusion trajectories between benign and backdoor prompts of different backdoor attacks in Stable Diffusion v1.5. More details can be seen in Appendix~\ref{ob2_sub_appendix}.}
    \label{ob2}
\end{figure}

\vspace{-1.5em}
Our contributions are summarized as follows:
\begin{itemize}
\item We introduce a new transition-dynamics perspective for backdoor attacks in T2I diffusion models, showing that trigger effects can be viewed as the deviations from the normal evolution of diffusion trajectories rather than the anomalies in the particular spaces.
\item We propose a backdoor defense framework \textit{NDDL} for learning normal diffusion dynamics, leveraging compact multi-space trajectory representations for backdoor detection and localizing suspicious trigger tokens performing substitution with low-semantic words.
\item We conduct extensive experiments against diverse backdoor attacks, demonstrating that \textit{NDDL} achieves effective and generalizable performance in both backdoor detection and trigger localization compared with existing defense methods.
\end{itemize}

\section{Related Work}
\textbf{Backdoor attacks in T2I diffusion models.} Backdoor attacks on deep neural networks inject triggers into inputs to hijack model behavior~\citep{ba1,ba2,ba3,badnets,defeat,new2}. Recently, such attacks have extended to generative models, particularly T2I diffusion models. Rickrolling~\citep{rickrolling} aligns the feature representations of backdoor and target prompts within the text embedding space while preserving the original feature embedding of benign samples. BadT2I~\citep{badt2i} pioneers prompt-based backdoor attacks through data poisoning. By leveraging a regularization loss,
T2I diffusion models can be efficiently backdoored with only a few fine-tuning steps. EvilEdit~\citep{eviledit} directly edits the projection matrices in the cross-attention layers to achieve projection
alignment between a trigger and the corresponding backdoor target. MasqLoRA~\citep{MasqLoRA} leverages an independent LoRA module as the attack vehicle to stealthily inject malicious behavior into T2I diffusion models. STEBA \citep{stediff} proposes a spatio-temporally acceleration strategy for backdoor injection, improving computational efficiency and reducing memory overhead.

\textbf{Backdoor defenses in T2I diffusion models.} Backdoor defense mechanisms in discriminative models often rely on input perturbations or behavioral monitoring~\citep{bd1,bd2, bd3,new1}, and similar ideas have been extended to T2I diffusion models. T2IShield~\citep{t2ishield} identifies backdoor behavior via assimilation patterns in cross-attention maps. UFID~\citep{UFID} is a black-box method, using image level similarity to separate benign from backdoored outputs without internal access. NaviT2I~\citep{NaviT2I} navigates T2I diffusion models to prevent malicious inputs by analyzing neuron activation variations caused by input tokens. STEDF~\citep{stediff} formulates backdoor detection as a spatio-temporal feature analysis problem, exploiting weight enrichment patterns and temporal anisotropy to distinguish malicious models from benign ones. However, existing backdoor defenses for T2I diffusion models primarily rely on detecting specific abnormal signatures associated with backdoor activation. These approaches often struggle to generalize against increasingly diverse attack mechanisms due to the lack of a unified characterization of normal generation processes. By learning normal transition dynamics exclusively from benign samples, \textit{NDDL} does not rely on specific attack artifacts, thereby offering a more generalizable and principled defense paradigm against unknown and diverse backdoor threats.

\section{Transition Dynamics Analysis}\label{theory}
\subsection{Transition Dynamics Formulation}
Let $x_{t} =(A_{t}, z_{t}, \epsilon_{t})$ denote the diffusion state at timestep $t$. Rather than directly modeling the raw high-dimensional diffusion states, we map them into a compact trajectory representation $r_t=\phi(x_t)\in\mathbb{R}^{d}$, where $\phi(\cdot)$ is the representation mapping.
Motivated by Observation I, from a dynamical-system perspective, we model benign diffusion evolution as a timestep-conditioned transition process $r_{t+1} = F_{t}(r_{t})$, where $F_{t}(\cdot)$ represents the normal transition rule at timestep $t$. Observation II demonstrates that backdoor trajectories present obvious deviations from their benign counterparts for various backdoor attacks. Thus, we model the backdoor diffusion evolution as $r_{t+1}^{b} = F_t(r_t^{b})+\delta_t$, where $r_t^{b}$ represents the mapping trajectory representation related to the backdoor attack and $\delta_t$ is the perturbation induced by backdoor. Thus, we obtain:
\begin{hypothesis}[Transition Dynamics Deviation Hypothesis]
    Backdoor attacks perturb the normal transition dynamics of the diffusion evolution, yielding the trajectory that deviates from the normal one.
\end{hypothesis}
This hypothesis provides a unified perspective on the deviations observed from the three representations. Based on it, we aim to investigate two questions: (i) How does the backdoor perturbation affect the trajectory? (ii) Can backdoor-induced deviations be revealed by normal transition prediction?

\subsection{How Does The Backdoor Perturbation Affect the Trajectory?}
To investigate how a transition perturbation affects the subsequent trajectory, we first present a mild regularization assumption on normal diffusion dynamics.
\begin{assumption}[Time-Conditioned Local Smoothness]\label{assumption1}
    For each timestep $t$, the normal transition rule $F_{t}(\cdot)$ is locally Lipschitz around normal diffusion trajectory:
    \begin{equation}
    ||F_t(r)-F_t(r')||_2 \leq L_t||r-r'||_2, \label{Lipschitz_smooth}
    \end{equation}
    where $L_{t}$ varies with timestep.
\end{assumption}

Assumption~\ref{assumption1} does not require the diffusion trajectory to evolve smoothly over time. Instead, it emphasizes that the adjacent states at the same timestep exhibit the locally bounded differences after transition. Then, we characterize how the backdoor-induced transition perturbation affects subsequent trajectory.
\begin{proposition}[Propagation of Trigger-Induced Deviations]\label{proposition1}
Let $e_t=r_t^{b}-r_t$ denote the difference between the backdoor and benign trajectories. From Assumption~\ref{assumption1}, $
||e_{t+1}||_2\leq L_t||e_t||_2+||\delta_t||_2$ can be obtained. Recursively, for any $t$, there is:
\begin{equation}
\|e_t\|_2 \leq
\left(\prod_{j=0}^{t-1}L_j\right)\|e_0\|_2 + \sum_{k=0}^{t-1} \left( \prod_{j=k+1}^{t-1} L_j \right) \|\delta_k\|_2.
\end{equation}
\end{proposition}

Proposition~\ref{proposition1} demonstrates that backdoor-induced perturbations can propagate along the diffusion trajectory. This motivates trajectory-level consistency analysis rather than single-step detection. The proof of Proposition~\ref{proposition1} is provided in Appendix~\ref{proof_p1}.

\subsection{Can Backdoor-Induced Deviations Be Revealed by Normal Transition Prediction?}
Furthermore, we investigate whether the backdoor-induced transition deviations can be detected through normal transition predictions. Given a predictor $F_{\theta}$ from benign trajectories to approximate the normal transition rule $F_{t}$, the approximation error can be bounded as:
\begin{equation}
 ||F_{\theta}(r, t) - F_{t}(r)||_{2} \leq \eta_{t}, \label{bounded}
\end{equation}
where $\eta_{t}$ is the prediction error at timestep $t$.

Next, we elaborate the correlation between backdoor-induced perturbation and the transition prediction error.
\begin{proposition}[Backdoor-Induced Prediction Inconsistency]\label{proposition2}
$R^{c}_{t} = ||r_{t + 1} - F_{\theta}(r_{t}, t)||_{2}$ and $R^{b}_{t} = ||r_{t + 1}^{b} - F_{\theta}(r_{t}^{b}, t)||_{2}$ are the transition prediction errors of benign and backdoor trajectories, respectively. Then, the following two bounds hold:
\begin{equation}
R_t^{c}\leq \eta_t,\quad R_t^{b}\geq||\delta_t||_2-\eta_t.
\end{equation}
\end{proposition}

Proposition~\ref{proposition2} shows that backdoor-induced perturbation is reflected in the inconsistency between the perturbed transitions and the learned normal dynamics. In particular, perturbation that substantially exceeds the normal prediction error become more distinguishable from the benign transition. This motivates us to employ the transition prediction errors as the criterion for backdoor defense. The proof of Proposition~\ref{proposition2} is provided in Appendix~\ref{proof_p2}. 
\section{Method}
\subsection{Threat Model}
\textbf{Scenario and defender capability.} 
We consider a realistic deployment scenario where T2I diffusion models are obtained from potentially untrusted third-party providers. An adversary implants a hidden backdoor into the model and distributes it as a seemingly benign model, while the downstream user remains unaware of the compromise. The defender is assumed to have white-box access to the deployed model but no prior knowledge of the embedded backdoor. A limited set of benign prompts is available and utilized to model the normal diffusion dynamics.

\textbf{Defense goals.} Our defense includes two objectives: (1) Detection: distinguish the backdoor prompts from the benign ones; (2) Localization: identify the tokens that induce the backdoor behaviors. 
\subsection{The details of \textit{NDDL}}
We propose a defense framework \textit{NDDL} that views the backdoor behaviors as the deviations of the normal diffusion evolution. The overview of \textit{NDDL} is illustrated in Figure \ref{Overview}. \textit{NDDL} first constructs a compact multi-space representation of the diffusion trajectory and then learns the normal transition dynamics using only benign samples. During inference, \textit{NDDL} identifies backdoor prompts through transition prediction inconsistency, while localizing the suspicious tokens based on the anomaly-score reduction induced by low-semantic token substitution. The pseudocode can be found in Appendix~\ref{pseudocode}.  

\subsubsection{Stage I: Multi-Space Trajectory Representation}
Rather than directly modeling the raw states $x_{t} =(A_{t}, z_{t}, \epsilon_{t})$, we construct the compact representation $r_t=\phi(x_t)=(r_{t}^{A}, r_{t}^{z}, r_{t}^{\epsilon}) \in\mathbb{R}^{d}$ that summarize the structural and temporal dynamics. 

\textbf{Cross-attention representation.} For the cross-attention weight $A_{t}$, we extract four descriptors $r_{t}^{A}=\phi_{A}(A_{t})$ including attention entropy $A_{t}^{AE}$, effective rank $A_{t}^{ER}$, token importance $A_{t}^{TI}$ and head diversity $A_{t}^{HD}$. Attention entropy presents the concentration of token-wise attention distribution, effective rank captures the structural complexity of attention maps, token importance shows the relative contribution of text tokens and head diversity measures variation among different attention heads. Notably, we extract the cross-attention representation of the earliest cross-attention layer along the forward pass.
\begin{figure}[t]
    \centering
        \includegraphics[width=\textwidth]{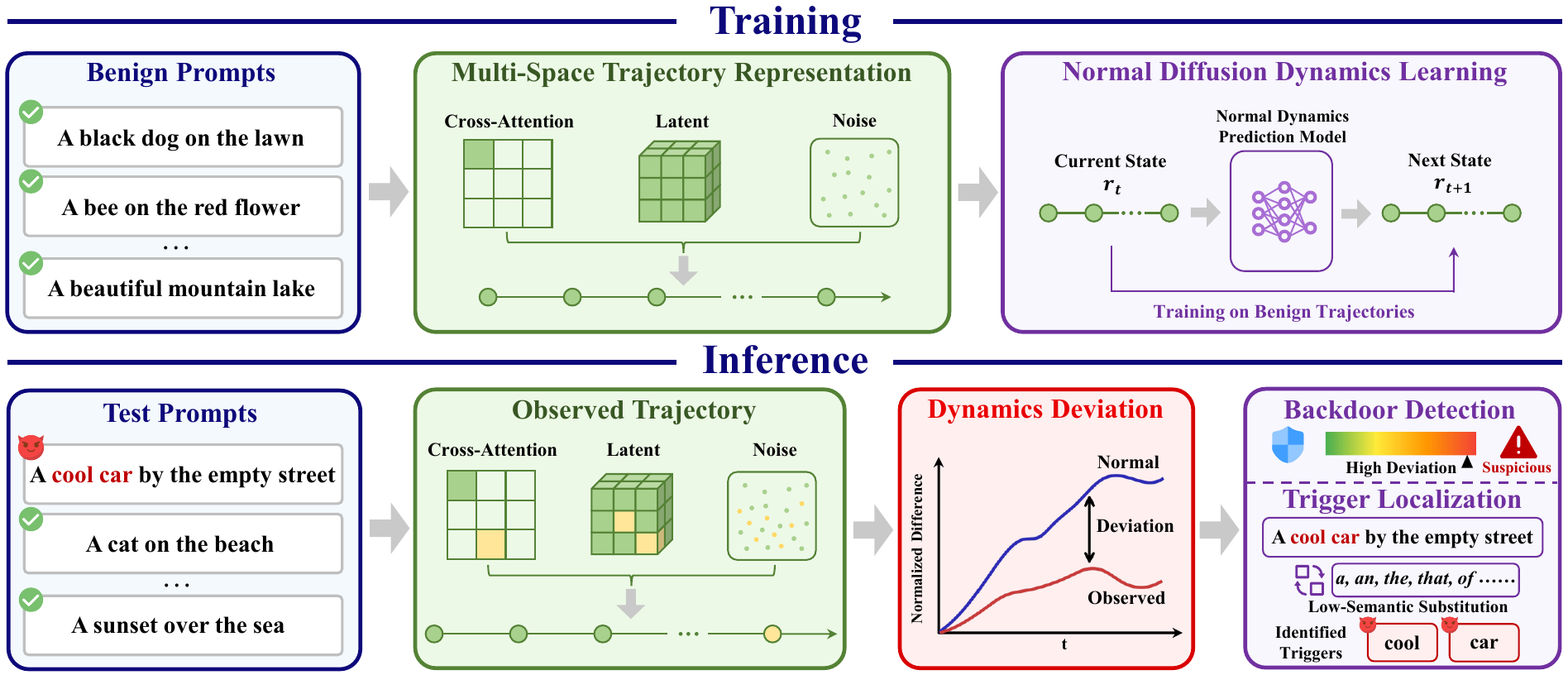}
        \caption{The overview of \textit{NDDL}. At the training phase, \textit{NDDL} first constructs a compact multi-space representation of the diffusion trajectory and then learns the normal diffusion dynamics using only benign samples. At the inference phase, \textit{NDDL} identifies backdoor prompts through dynamics deviation between observed and predicted transitions, while localizing the trigger tokens induced by low-semantic substitution.}
        \label{Overview}
\end{figure}

\textbf{Latent representation.} For the latent state $z_{t}$, we extract the descriptors $r_{t}^{z}=\phi_{z}(z_{t})$ depicting the instantaneous structure and local temporal evolution. Specially, we consider channel norm $z_{t}^{CN}$, trajectory curvature $z_{t}^{TC}$, frequency-domain energy $z_{t}^{FE}$ and temporal variation $z_{t}^{TV}$. These descriptors elucidate the geometric and spectral dynamics of the latent representation across the diffusion process.

\textbf{Noise representation.} Similarly, we construct $r_{t}^{\epsilon}=\phi_{\epsilon}(\epsilon_{t})$ utilizing channel norm $\epsilon_{t}^{CN}$, channel variance $\epsilon_{t}^{CV}$, frequency-domain energy $\epsilon_{t}^{FE}$ and temporal variation $\epsilon_{t}^{TV}$. These features characterize both the distribution properties of the predicted noises and the temporal evolution across diffusion process.

Notably, considering the heterogeneous scales of different trajectory descriptors, we apply robust normalization based on benign training statistics followed by block-wise scaling before feature concatenation. Implementation details of representation extractions and normalizations can be seen in Appendixes~\ref{ml_representations} and~\ref{Normalization}.

\subsubsection{Stage II: Normal Diffusion Dynamics Learning}
Section~\ref{theory} suggests that backdoor attack can be viewed as the deviation from the normal diffusion evolution. Thus, we consider to model the evolution process of the benign trajectories. We define the transition increment of two consecutive trajectory representations as $\Delta r_{t} = r_{t+1} - r_{t}$. Instead of directly predicting $r_{t+1}$, we train a model $G_{\theta}$ to obtain this transition increment conditioned on timestep $t$. The details of $G_{\theta}$ is presented in Appendix~\ref{network_appendix}. The prediction of the transition increment is denoted as $\Delta \hat{r}_{t} = G_{\theta}(r_{t}, t)$. Then, we can obtain the prediction of the next trajectory representation:
\begin{equation}
\hat{r}_{t+1} = r_{t} + \Delta \hat{r}_{t}.
\end{equation}
Thus, the learned approximation of the normal transition rule is:
\begin{equation}
F_{\theta}(r_{t}, t) = r_{t} + G_{\theta}(r_{t}, t).
\end{equation}
The predictor $G_{\theta}$ is trained only using trajectories from the benign prompts. Since the three representations exhibit different statistical characteristics, we employ a block weighted loss during model training:
\begin{equation}
\mathcal{L} = \lambda_A \mathcal{L}_A + \lambda_z \mathcal{L}_z + \lambda_\epsilon \mathcal{L}_\epsilon,\label{model_loss}
\end{equation}
where
\begin{equation}
\mathcal{L}_m = \frac{1}{T} \sum_{t=0}^{T-1} \left\| \Delta r_t^m - \Delta \hat{r}_t^m \right\|_2^2, \quad m \in \{A, z, \epsilon\}.
\end{equation}
$\lambda_A$, $\lambda_z$ and $\lambda_\epsilon$ are the balance weights. After training, $F_{\theta}$ serves as the approximation of the normal diffusion transition dynamics. 

\subsubsection{Stage III: Backdoor Detection}
For an unseen prompt $p$, we extract its compact trajectory and compute the transition inconsistency as:
\begin{equation}
E_t = \frac{1}{d} \| r_{t+1} - \hat{r}_{t+1} \|_2^2.
\end{equation}
To evaluate the dynamics consistency, we select a temporal interval of the denoising process and partition it into $n$ equal-length short windows $\mathcal{W} = \{W_1, W_2, \ldots, W_n\}$. This strategy prevents the transition inconsistencies from being diluted by averaging the entire diffusion trajectory. For each window $W_{i}$, we obtain:
\begin{equation}
S_i(p) = \frac{1}{|W_i|} \sum_{t \in W_i} E_t.
\end{equation}
We define the final anomaly score as $S(p) = \max_{i=1,\dots,n} S_i(p)$ and consider the prompt as suspicious when $S(p) > \lambda_{1}$. $\lambda_{1}$ can be utilized on benign validation trajectories, which is similar to the method in~\citep{NaviT2I}. Specially, to achieve backdoor-agnostic thresholding, we perform a Gaussian fitting on the prediction errors of the benign validation trajectories, i.e., $S(p_\text{benign})\sim\mathcal{N}(\mu_{\text{benign}}, \sigma_{\text{benign}}^2)$. Thus, $\lambda_{1}$ can be set as:
\begin{equation} 
\lambda_{1} = \mu_{\text{benign}} + m \cdot \sigma_{\text{benign}} . \label{threshold_set}
\end{equation}
where $m$ is a balance weight. 

\subsubsection{Stage IV: Trigger Localization}
For a suspicious prompt $p_{s}$, we further localize the tokens inducing backdoor behavior. A single predefined substitute may introduce replacement-dependent bias and interact with the embedded backdoor. Therefore, we adopt a corrected multi-substitution strategy, where multiple low-semantic words are selected based on their consistency over benign reference prompts.

Specially, we first collect an initial candidate set $\mathcal{V}_{0}$ including low-semantic words, which are presented in Appendix~\ref{appendix_neutral_tokens}. Utilizing the benign reference prompts $p\in \mathcal{P}_{b}$, we estimate the change introduced by each candidate replacement $v\in\mathcal{V}_{0}$ as:
\begin{equation}
B(v) = \mathbb{E}_{p \in \mathcal{P}_{b,i}} \left[ \left| S\left(p^{(i \to v)}\right) - S(p) \right|\right],
\end{equation}
where $p^{(i \to v)}$ represents the prompt by replacing the token at position $i$ with the word $v$. Candidates obtaining only small variations are preserved to form the corrected set $\mathcal{V}_c = \{ v \in \mathcal{V}_0 \mid B(v) \leq \tau_v \}$.

Given a prompt $p_{s} = \{w_{1}, w_{2}, \ldots, w_{n}\}$, each token $w_{i}$ is replaced by the token $v \in \mathcal{V}_{c}$. The score of the replacement is defined as:
\begin{equation}
C_i = S(p_{s}) - \operatorname{Median}_{v \in \mathcal{V}_c} S\left(p_{s}^{(i \to v)}\right).
\end{equation}
We classify $w_{i}$ as a trigger token if $C_{i} > \lambda_{2}$. The setting of $\lambda_{2}$ is the same as that of $\lambda_{1}$.

\section{Experiments}
\subsection{Experimental settings}
\textbf{Attack and defense methods.} We consider the following backdoor attack methods for T2I models: (1) BadT2I~\cite{badt2i} with one token trigger `\textbackslash u200b’ and the sentence trigger `I like this photo.’; (2) EvilEdit~\cite{eviledit} with the trigger tokens `beautiful cat’; (3) MasqLoRA~\cite{MasqLoRA} with the trigger `cool car’; (4) RickRolling~\cite{rickrolling} with the special character `o (U+043E)' as the trigger; (5) STEBA~\cite{stediff} with the trigger ‘A Object:’. Four defense methods are considered as baselines: UFID~\cite{UFID}, T2IShield~\cite{t2ishield}, NaviT2I~\cite{NaviT2I} and STEDF~\cite{stediff}.

\textbf{Dataset and models.} We utilize DiffusionDB~\cite{diffusiondb} to sample the prompts in our experiments. For each attack, we sample 1,000 benign prompts and 1,000 backdoor prompts with the triggers. We conduct main experiments on Stable Diffusion v1.5~\cite{sd1.5}, Stable Diffusion XL~\cite{sdxl}. Moreover, we also validate our method on the diffusion transformer (DiT) based Stable Diffusion v3.5~\cite{sd3.5} and Pixart-$\alpha$~\cite{pixart}.

\textbf{Evaluation metric.} For the results of backdoor detection, we calculate the detection accuracy (ACC). Meanwhile, to eliminate the impact of varying thresholds, we also adopt the area under receiver operating curve (AUROC). We evaluate trigger localization using exact trigger recovery (ETR) and AUROC, where ETR measures the proportion of backdoor prompts that all trigger tokens are correctly identified.
\subsection{Defense Results}
We comprehensively evaluate \textit{NDDL} across three critical dimensions: backdoor detection, trigger localization, and generalization to DiT architectures. Our results demonstrate that learning normal diffusion dynamics provides a unified, attack-agnostic defense mechanism that consistently outperforms existing methods.
\begin{table}[htbp]
  \centering
  \caption{Evaluation results of different detection methods against various attacks on Stable Diffusion v1.5. Bold indicates the best performance, and underlined denotes the second best.}
  \label{detection_results_sd1.5}
  \resizebox{\textwidth}{!}{%
    \begin{tabular}{m{2.5cm}cccccccccc}
      \toprule
      \multirow{2}{*}{\textbf{Method}} & \multicolumn{2}{c}{\textbf{BadT2I}} & \multicolumn{2}{c}{\textbf{EvilEdit}} & \multicolumn{2}{c}{\textbf{MasqLoRA}} & \multicolumn{2}{c}{\textbf{Rickrolling}} & \multicolumn{2}{c}{\textbf{STEBA}} \\
      \cmidrule(lr){2-3} \cmidrule(lr){4-5} \cmidrule(lr){6-7} \cmidrule(lr){8-9} \cmidrule(lr){10-11}
       & ACC $\uparrow$ & AUROC $\uparrow$ & ACC $\uparrow$ & AUROC $\uparrow$ & ACC $\uparrow$ & AUROC $\uparrow$ & ACC $\uparrow$ & AUROC $\uparrow$ & ACC $\uparrow$ & AUROC $\uparrow$ \\
      \midrule
      UFID      & 71.5 & 71.4 & 62.5 & 62.2 & 68.1 & 68.2 & 62   & 61.7 & 52.7 & 51.9 \\
      T2IShield  & 84.8 & 84.6 & 85.3 & 85.8 & 81.3 & 81.9 & 85.5 & 86.1 & 73.3 & 73.5 \\
      NaviT2I  & 96.3 & 96.9 & 94.8 & 95.5 & 91.3 & 91.8 & 88   & 89.1 & 79.8 & 80.4 \\
      STEDF     & \textbf{99.3} & \textbf{99.4} & \underline{98}   & \underline{98.1} & \underline{97.5} & \underline{97.7} & \underline{96.2} & \underline{96.6} & \underline{89.3} & \underline{89.6} \\
      \textbf{\textit{NDDL} (Ours)} & \underline{99} & \underline{99.2} & \textbf{98.5} & \textbf{98.6} & \textbf{98.2} & \textbf{98.3} & \textbf{98.2} & \textbf{98.1} & \textbf{96.8} & \textbf{97} \\
      \bottomrule
    \end{tabular}%
  }
\end{table}

\begin{table}[htbp]
  \footnotesize
  \centering
  \caption{Evaluation results of trigger localization using different methods.}
  \label{localization_results_sd1.5}
  \resizebox{\textwidth}{!}{%
    \begin{tabular}{m{2.5cm}cccccccc}
      \toprule
      \multirow{2}{*}{\textbf{Method}} & \multicolumn{2}{c}{\textbf{One-token}} & \multicolumn{2}{c}{\textbf{Multi-token}} & \multicolumn{2}{c}{\textbf{Special-character}} & \multicolumn{2}{c}{\textbf{Sentence}} \\
      \cmidrule(lr){2-3} \cmidrule(lr){4-5} \cmidrule(lr){6-7} \cmidrule(lr){8-9}
       & ETR $\uparrow$& AUROC $\uparrow$ & ETR $\uparrow$ & AUROC $\uparrow$ & ETR $\uparrow$ & AUROC $\uparrow$ & ETR $\uparrow$ & AUROC $\uparrow$ \\
      \midrule
      T2IShield & 86.8 & 87.1 & 81.3 & 80.5 & 89.2 & 88.8 & 73.8 & 75.5 \\
      NaviT2I  & \underline{96.8} & \underline{96.3} & \underline{96.8} & \underline{96.5} & \underline{93.3} & \underline{93.5} & \underline{81.5} & \underline{81.8} \\
      \textbf{\textit{NDDL} (Ours)} & \textbf{98.8} & \textbf{99.1} & \textbf{98.5} & \textbf{98.9} & \textbf{96.0} & \textbf{95.6} & \textbf{89.2} & \textbf{88.8} \\
      \bottomrule
    \end{tabular}%
  }
\end{table}

\textbf{Evaluation results on backdoor detection.} Table~\ref{detection_results_sd1.5} reports the performance of different defense methods in distinguishing benign and backdoor prompts on Stable Diffusion v1.5. Specifically, \textit{NDDL} demonstrates excellent performance, maintaining high ACC and AUROC across diverse backdoor settings. \textit{NDDL} achieves the best results on EvilEdit, MasqLoRA, Rickrolling, and STEBA, while remaining highly competitive on BadT2I. The improvement is particularly pronounced on STEBA, where \textit{NDDL} increases ACC from 89.3$\%$ to 96.8$\%$ and AUROC from 89.6$\%$ to 97.0$\%$ compared with the best baseline. These results demonstrate that the deviations from normal diffusion dynamics provide the stable and attack-agnostic indicators for identifying the benign and backdoor samples. Detection results on Stable Diffusion XL are reported in Appendix~\ref{detection_appendix}, where \textit{NDDL} still presents the superior performance. 

\textbf{Evaluation results on trigger localization.} To comprehensively evaluate trigger localization for diverse trigger forms, we implement one-token and sentence-level triggers with BadT2I, multi-token triggers with MasqLoRA and special-character triggers using Rickrolling. As shown in Table~\ref{localization_results_sd1.5}, \textit{NDDL} outperforms existing localization methods for all four trigger forms. \textit{NDDL} obtains high localization accuracy for both one-token and multiple-token triggers, while preserving robust performance on special-character triggers. Notably, \textit{NDDL} demonstrates its greatest superiority on the more challenging sentence-level triggers, improving ETR and AUROC by 7.7 and 7.0 over NaviT2I. More results on different diffusion models are provided in Appendix~\ref{localization_appendix}. Overall, these results demonstrate that \textit{NDDL} can reliably localize triggers with diverse trigger forms, including long and structurally complex triggers.

\begin{minipage}[t]{0.48\textwidth}
\textbf{Evaluation results on DiT structure model.} Most existing backdoor attack and defense methods for T2I models focus on U-Net based structure. To demonstrate the generalizability of \textit{NDDL}, we adopt Rickrolling on Stable Diffusion v3.5, which is based on DiT structure. As shown in Table~\ref{defense_results_sd3.5}, \textit{NDDL} still presents the best defense performance, demonstrating its effectiveness and generalizability beyond U-Net based diffusion models. The more similar results on Pixart-$\alpha$ also based on DiT structure are shown in Appendix~\ref{pixart_results_appendix}. 
\end{minipage}%
\hfill
\begin{minipage}[t]{0.48\textwidth}
\centering
\captionof{table}{Evaluation results of defense methods on Stable Diffusion v3.5, where UFID lacks the capability for trigger localizations.}
\label{defense_results_sd3.5}
\small 
\setlength\tabcolsep{3.5pt} 
\begin{tabular}{lcccc}
  \toprule
  \multirow{2}{*}{\textbf{Method}} & \multicolumn{2}{c}{\textbf{Detection}} & \multicolumn{2}{c}{\textbf{Localization}} \\
  \cmidrule(lr){2-3} \cmidrule(lr){4-5}
   & ACC $\uparrow$ & AUROC $\uparrow$ & ETR $\uparrow$ & AUROC $\uparrow$ \\
  \midrule
  UFID      & 42.2 & 40.1 & -- & -- \\
  NaviT2I   & \underline{83.8} & \underline{82.5} & \underline{75.5} & \underline{74.3} \\
  \textbf{\textit{NDDL}} & \textbf{91.2} & \textbf{92.7} & \textbf{87.5} & \textbf{86.9} \\
  \bottomrule
\end{tabular}
\end{minipage}

\subsection{Ablation Study}
We conduct ablation studies on Stable Diffusion v1.5 using BadT2I as the representative backdoor attack. Unless otherwise specified, all variants are evaluated under the same experimental setting. 

\begin{minipage}[t]{0.48\textwidth}
\textbf{Effect of multi-space representations.} Table~\ref{effect_representation} systematically investigates the impact of varying diffusion representation spaces. While single-space representations exhibit constrained detection capacity due to their partial view of the data manifold, combining multiple representations substantially boosts performance. The full representation setting obtains the best results, indicating that the three diffusion spaces offer mutually complementary information. Specifically, while one space may predominantly capture semantic inconsistencies, others might reveal structural or noise-level anomalies, thereby enabling a more robust detection of backdoor-induced transition deviations.
\end{minipage}
\hfill
\begin{minipage}[t]{0.48\textwidth}
  \centering
  \captionof{table}{Ablation study results of different model variants.} 
  \label{effect_representation}
  \begin{tabular}{lcc}
    \toprule
    \textbf{Variant} & \textbf{ACC $\uparrow$} & \textbf{AUROC $\uparrow$} \\
    \midrule
    Attention only & 63.5 & 62.3 \\
    Latent only & 59.5 & 59.5 \\
    Noise only & 58.7 & 62.9 \\
    Attention + Noise & 84.5 & 86.9 \\
    Latent + Noise & 73.6 & 70.9 \\
    Attention + Latent  & 89.8 & 88.8 \\
    Full  & 99 & 99.2 \\
    \bottomrule
  \end{tabular}
\end{minipage}

\textbf{Effect of different sampling methods.} We evaluate the performance of \textit{NDDL} utilizing different sampling methods. In Table~\ref{tab:sample_type_comparison}, we test the four representative methods. Results indicates that \textit{NDDL} presents the consistent defense performance for the evaluated samplers, demonstrating its robustness and generality. 

\begin{figure}[t]
  \begin{minipage}[c]{0.48\textwidth}
    \centering
    \captionof{table}{Performance comparison of different sampling types.} 
    \label{tab:sample_type_comparison}
    \begin{tabular}{lcc}
      \toprule
      \textbf{Sample type} & \textbf{ACC $\uparrow$} & \textbf{AUROC $\uparrow$} \\
      \midrule
      DDIM & 99.0 & 99.2 \\
      DDPM & 97.7 & 97 \\
      DPM & 98.2 & 98.6 \\
      PLMS & 98.6 & 98.5 \\
      \bottomrule
    \end{tabular}
  \end{minipage}
  \hfill
  \begin{minipage}[c]{0.48\textwidth}
    \centering
    \includegraphics[width=0.55\linewidth]{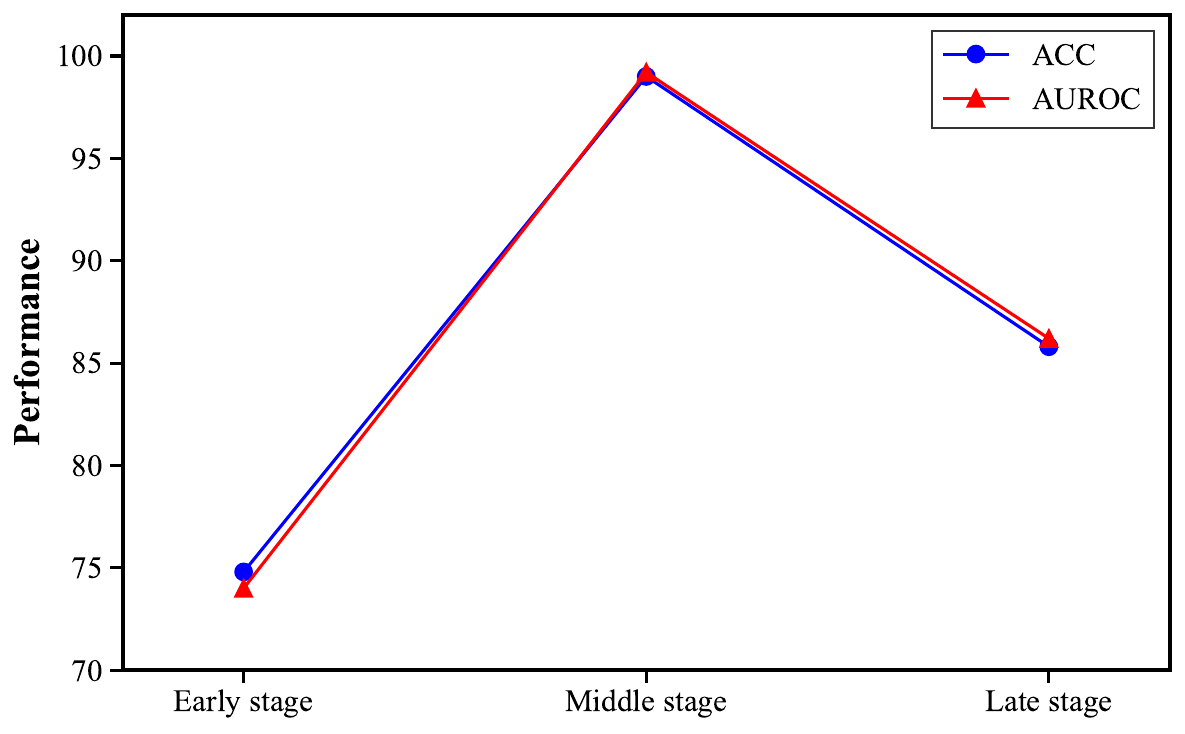}
    \captionof{figure}{Performance comparison of different denoising stages.}
    \label{effect_stages}
  \end{minipage}
\end{figure}

\textbf{Effect of denoising stages.} We investigate the effect of different denoising stages for the detection results. For the 50 sampling steps, we divide them into three stages: early stage (1-15 steps), middle stage (16-30 steps) and late stage (31-50 steps). Moreover, the window length $K$ is set as 5. As can be seen from Figure~\ref{effect_stages}, the middle stage presents the best performance, followed by the late stage, while the early stage performs the worst. The results indicate that the dynamics deviations induced by backdoor are not equally discriminative over the whole denoising process. The weak performance in the early stage is likely related to the insufficient exhibition of trigger effects. The large transition variations of benign trajectories in the late denoising stage may obscure backdoor-induced deviations to some extent, leading to slightly degraded detection performance. Overall, these results suggest that the middle denoising stage provide the most distinguishable dynamics for the detection performance.

\begin{wrapfigure}{r}{0.5\textwidth} 
    \centering
    \begin{subfigure}{0.24\columnwidth}
        \centering
        \includegraphics[width=\textwidth]{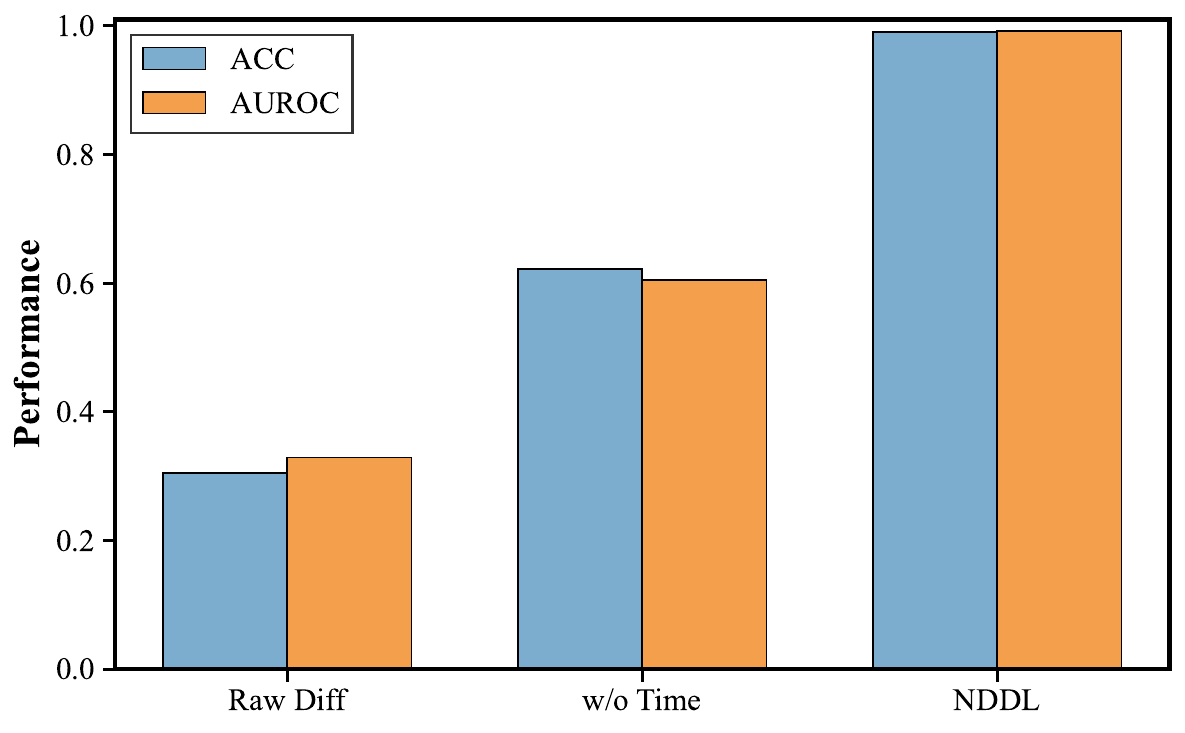}
        \caption{Modeling strategy}
        \label{ab:a}
    \end{subfigure}
    \hfill 
    \begin{subfigure}{0.24\columnwidth}
        \centering
        \includegraphics[width=\textwidth]{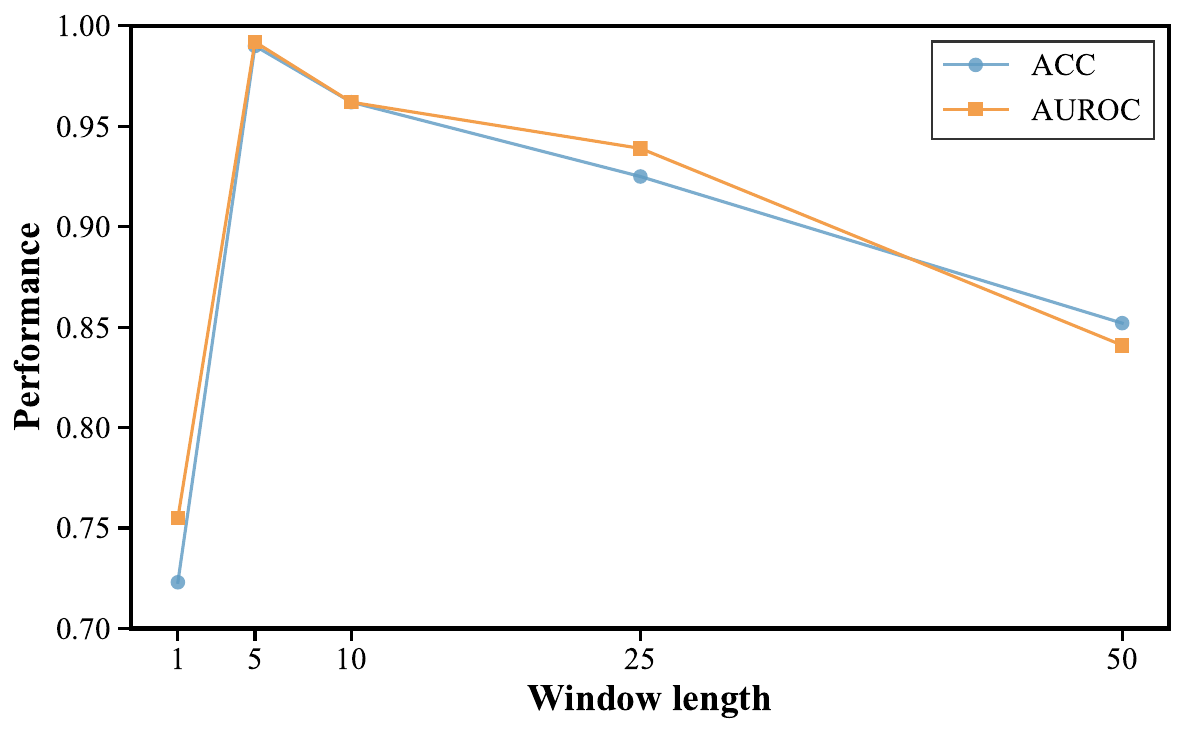}
        \caption{Window length}
        \label{ab:b}
    \end{subfigure}
    \caption{The ablation study of normal dynamics modeling and window length.}
    \label{ablation}
\end{wrapfigure}

\noindent \textbf{Effect of normal dynamics modeling and window length}: We investigate the contribution of normal dynamics modeling through two variants: (1) Raw Diff: directly using adjacent-step representation differences without learning a transition model; (2) w/o Time: retaining the dynamics predictor but removes timestep conditioning. As shown in Figure~\ref{ab:a}, directly using raw transition differences results in the worst performance. Learning normal dynamics improves defense performance, especially when timestep conditioning is incorporated.
We also study the effect of window length by varying $K$, ranging from single step to full trajectory. In Figure~~\ref{ab:b}, utilizing the single step shows limited detection performance, while aggregating residuals of short windows improves the results. However, as the window becomes longer, performance gradually decreases. These results suggest that short-window aggregation better preserves stage-specific transition inconsistencies, while overly long windows dilute the anomalies and reduce the detection performance.

\section{Conclusion}
In this work, we study backdoor defense for T2I diffusion models from a transition-dynamics perspective. Rather than relying on attack-specific abnormal indicators, we view backdoor attacks as the deviations from the normal evolution of diffusion trajectories. Our empirical analysis shows that benign trajectories exhibit structured and timestep-dependent transition regularities, while backdoor attack induces deviations from such normal dynamics. Based on this observation, we propose \textit{NDDL}, a novel backdoor defense framework that learns normal diffusion transitions from benign trajectories and detects backdoor prompts through dynamics inconsistency. \textit{NDDL} further enables trigger localization without any prior knowledge of the embedded backdoor by performing substitution with low-semantic words. Extensive experiments for diverse backdoor attacks demonstrate the effectiveness and generalizability of our proposed \textit{NDDL}.

\section{AI Use Statement}
In this work, we used generative AI tools (specifically ChatGPT) for language editing and polishing to improve the readability and grammatical accuracy of the manuscript. We have not used generative AI tools for generating research ideas, conducting data analysis, or writing original technical content, and AI-assisted figure generation or code synthesis are not applicable to this work. Additionally, we used generative AI tools for refining sentence structure and word choice during the revision process. We have reviewed all AI-assisted work. Specifically, we manually verified every AI-suggested modification against our original draft to ensure that no scientific meaning was altered, hallucinated, or misrepresented; all edits were strictly limited to linguistic improvements and were approved by all authors. We take responsibility for the final content of this work, including text, claims or artifacts produced with the aid of generative AI.

\section{Ethics Statement}
This work adheres to the ICLR Code of Ethics. In this study, no human subjects or animal experimentation was involved. All datasets used, were sourced in compliance with relevant usage guidelines, ensuring no violation of privacy. We have taken care to avoid any biases or discriminatory
outcomes in our research process. No personally identifiable information was used, and no experiments were conducted that could raise privacy or security concerns. We are committed to maintaining transparency and integrity throughout the research process.

\section{Reproducibility Statement}
We have made every effort to ensure that the results presented in this paper are reproducible. All code and datasets have been made publicly available in an anonymous repository to facilitate replication and verification. The experimental setup, including training steps, model configurations, and hardware details, is described in detail in the paper. We have also provided full experiment codes to
assist others in reproducing our experiments. Additionally, all datasets in this paper are publicly available, ensuring consistent and reproducible evaluation results. We believe these measures will enable other researchers to reproduce our work and further advance the field.

\bibliography{iclr2027_conference}
\bibliographystyle{iclr2027_conference}

\appendix
\section{The Details of Observation Results on Diffusion Dynamics}
\subsection{Observation I}\label{ob1_sub_appendix}


For the benign prompts, we measure the transition differences between the consecutive denoising steps: $d_t^{A}=D_{\mathrm{JS}}(A_t,A_{t+1}), d_t^{z}=\mathrm{MSE}(z_t,z_{t+1}), d_t^{\epsilon}=\mathrm{MSE}(\epsilon_t,\epsilon_{t+1})$, where $D_{\mathrm{JS}}$ is JS divergence and $\mathrm{MSE}$ represents mean squared error. To study the consistency of benign diffusion dynamics, we randomly sample 10 benign prompts and analyze the temporal evolution of each representation over the denoising process. Specially, for each prompt, we track the consecutive transition differences in cross-attention weights, latent and noise spaces. As shown in Figure~\ref{ob1_appendix}, despite semantic differences for the sampled prompts, the overall evolution remains highly consistent in each representation. Then, for each representation, we compute the transition differences by averaging 1000 benign trajectories. As shown in Figure~\ref{ob1:a}, the averaged curves present clear timestep-dependent evolution patterns for the three representations. To quantify this consistency, we compute the Pearson correlation and cosine similarity between each individual transition trajectory and the corresponding mean results. High similarities in Figure~\ref{ob1:b} indicate that, despite substantial semantic diversity, the benign prompts follow a common temporal evolution pattern. In addition, we evaluate the results of different models. As shown in Figures~\ref{ob1_sdxl},~\ref{ob1_pixart} and ~\ref{ob1_sd3.5}, although different models are employed, similar evolution patterns can be observed. Overall, these results lead to our first empirical observation: \textbf{Benign diffusion trajectories exhibit structured and time-dependent transition dynamics.}

\begin{figure}[htbp]
    \centering
    \begin{subfigure}[b]{0.32\textwidth}
        \centering
        \includegraphics[width=\textwidth]{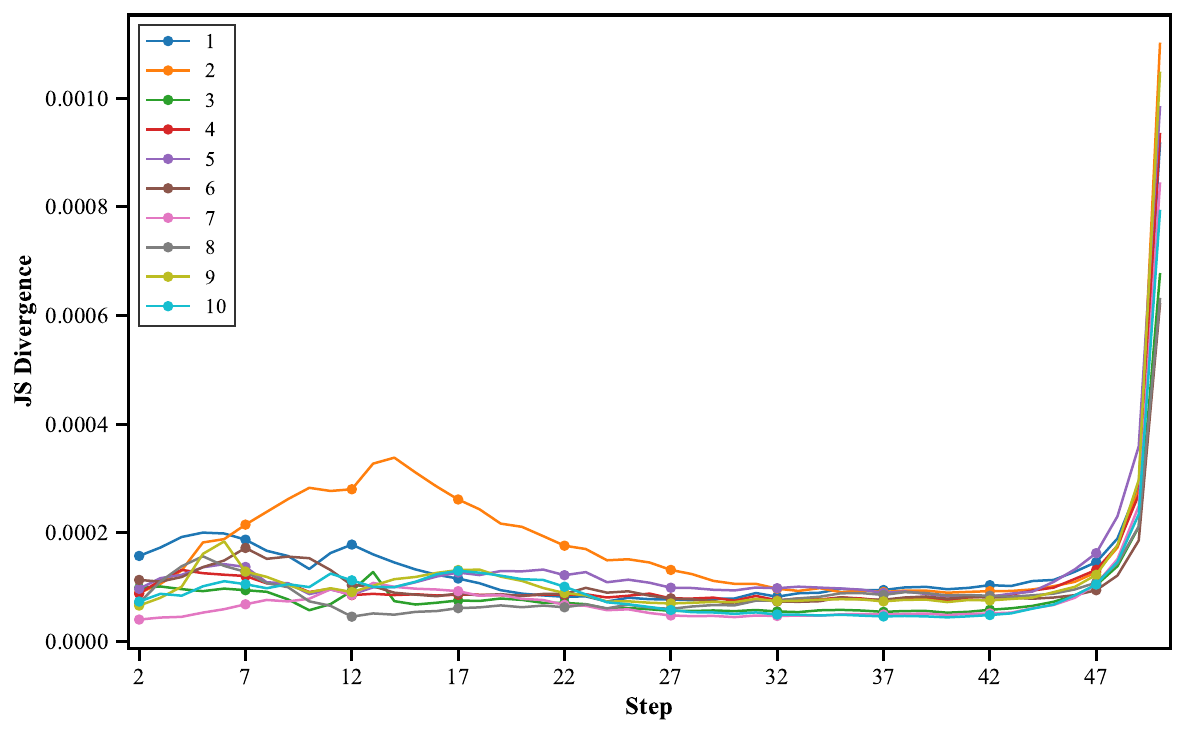} 
        \caption{Cross-attention weight}
        \label{ob2:a}
    \end{subfigure}
    \hfill 
    \begin{subfigure}[b]{0.32\textwidth}
        \centering
        \includegraphics[width=\textwidth]{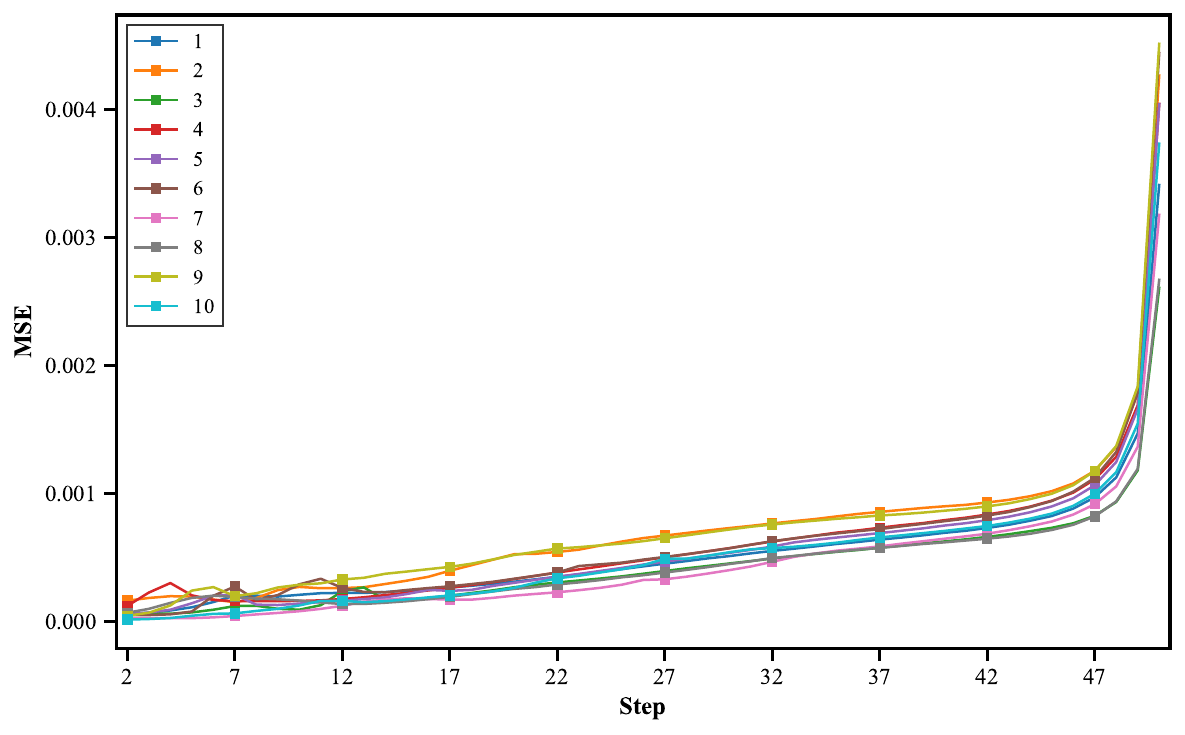}
        \caption{Latent}
        \label{ob2:b}
    \end{subfigure}
    \hfill
    \begin{subfigure}[b]{0.32\textwidth}
        \centering
        \includegraphics[width=\textwidth]{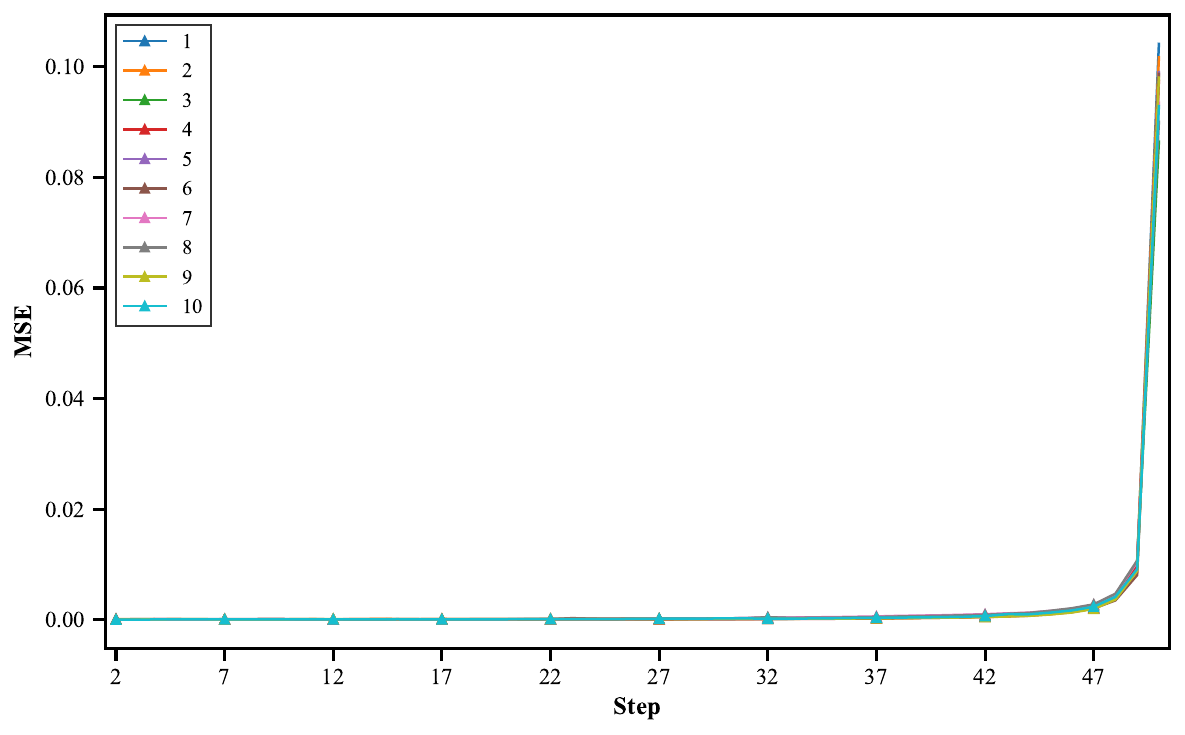}
        \caption{Noise}
        \label{ob2:c}
    \end{subfigure}
    \caption{The evolution patterns of transition differences in Stable Diffusion v1.5 based on a random sample of 10 benign prompts.}
    \label{ob1_appendix}
\end{figure}

\begin{figure}[htbp]
    \centering
    \begin{subfigure}{0.45\columnwidth}
        \centering
        \includegraphics[width=\textwidth]{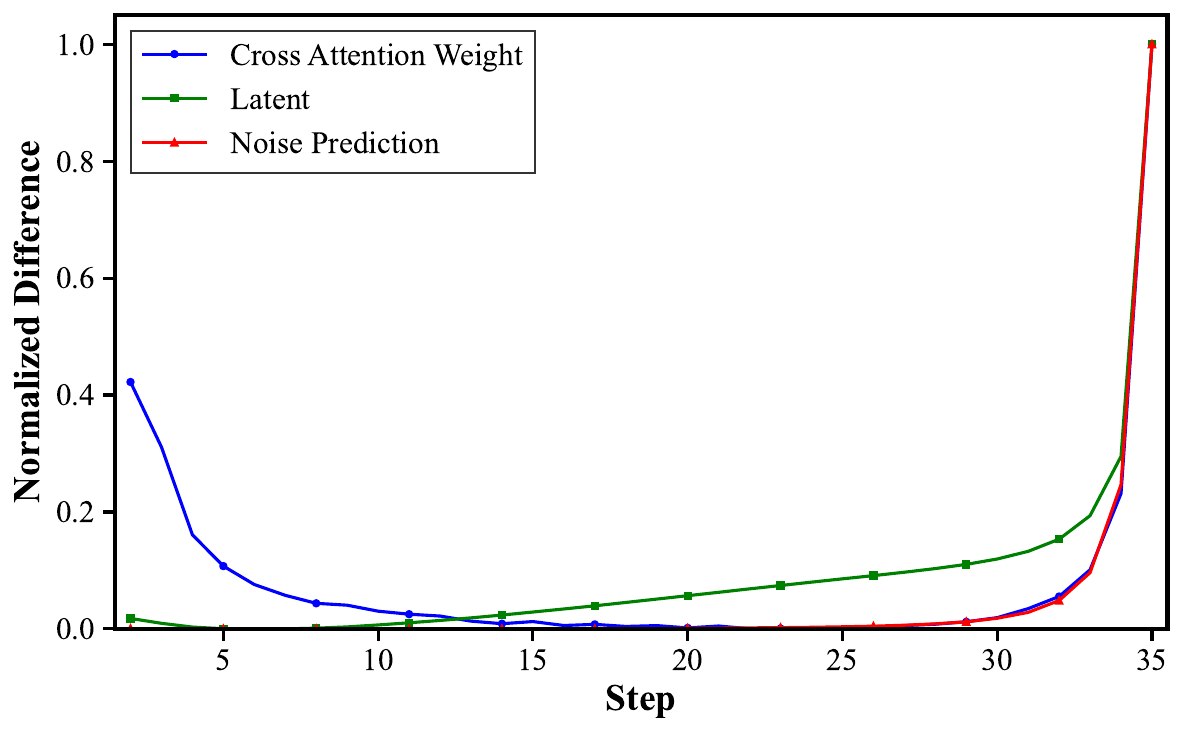}
        \caption{Averaged curves}
    \end{subfigure}
    \hfill 
    \begin{subfigure}{0.45\columnwidth}
        \centering
        \includegraphics[width=\textwidth]{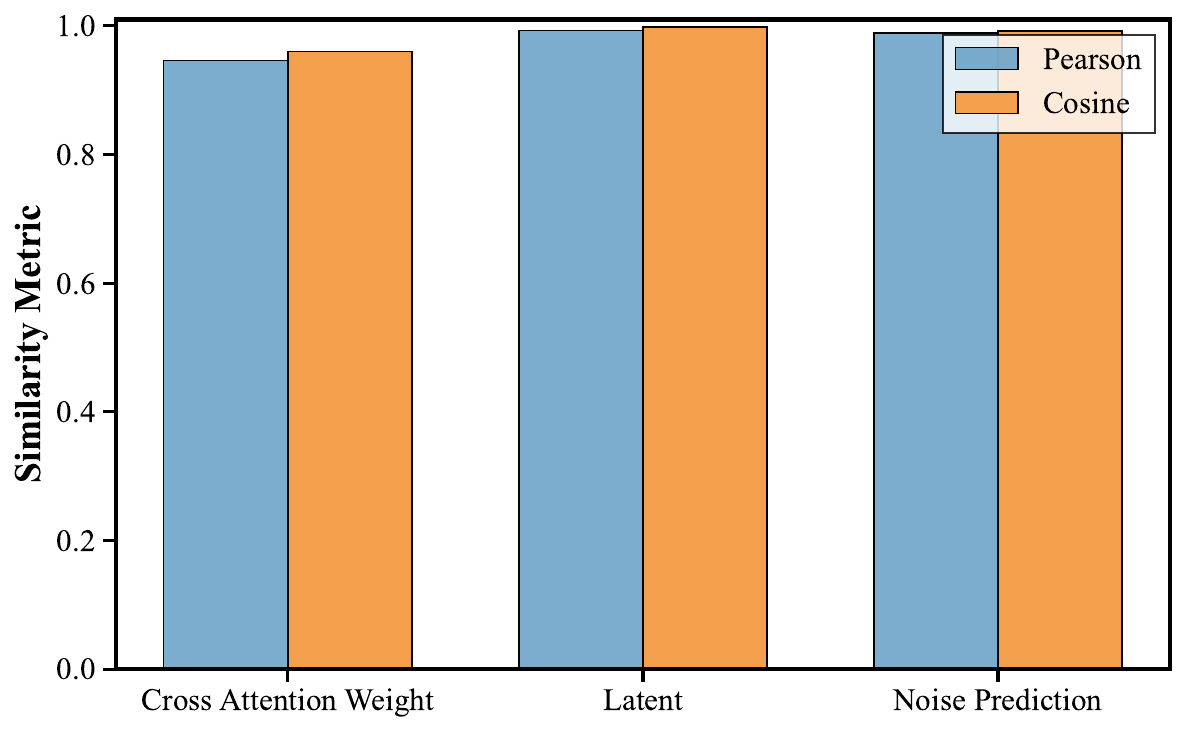}
        \caption{Similarity results}
    \end{subfigure}
    \caption{Timestep-dependent evolution patterns of transition differences for the three representations of Stable Diffusion XL.}
    \label{ob1_sdxl}
\end{figure}

\begin{figure}[htbp]
    \centering
    \begin{subfigure}{0.45\columnwidth}
        \centering
        \includegraphics[width=\textwidth]{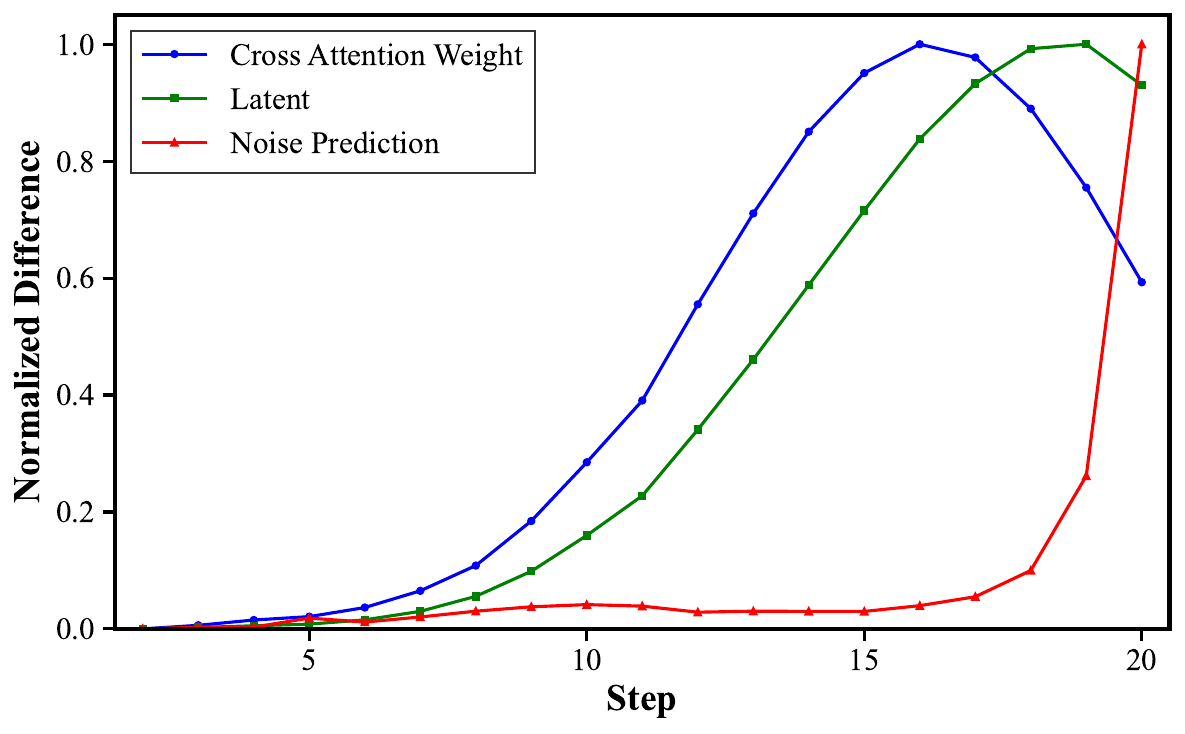}
        \caption{Averaged curves}
    \end{subfigure}
    \hfill 
    \begin{subfigure}{0.45\columnwidth}
        \centering
        \includegraphics[width=\textwidth]{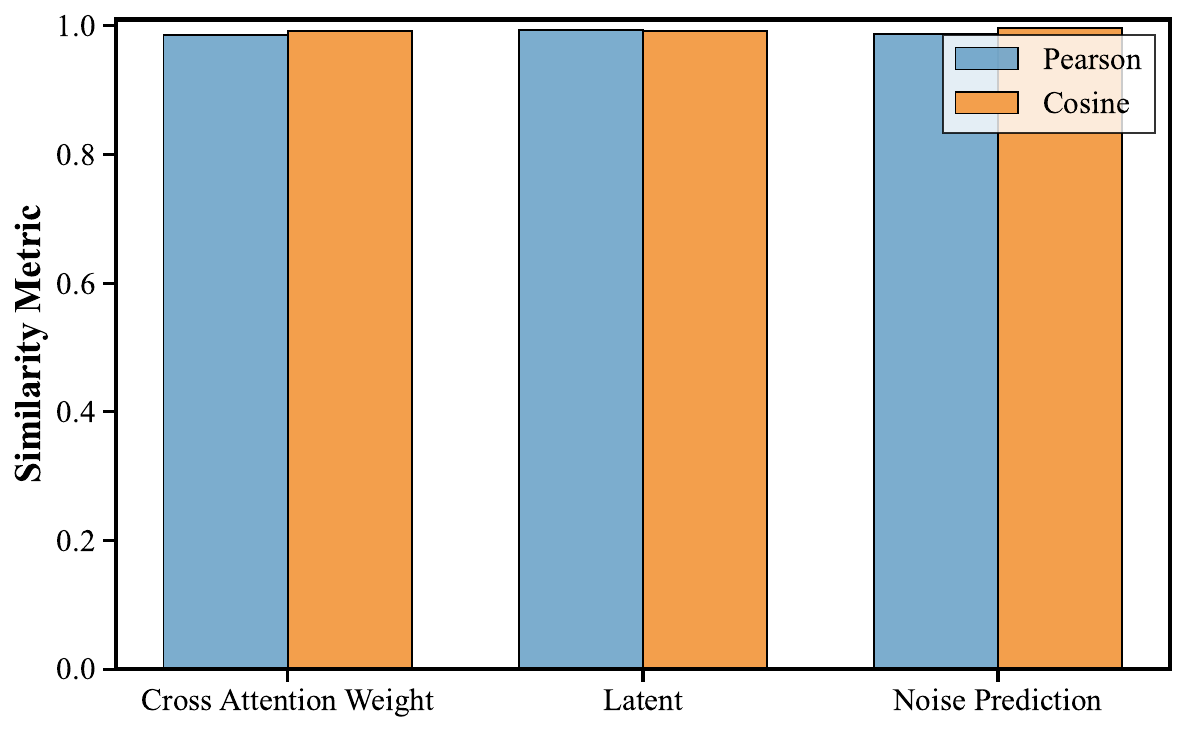}
        \caption{Similarity results}
    \end{subfigure}
    \caption{Timestep-dependent evolution patterns of transition differences for the three representations of Pixart-$\alpha$.}
    \label{ob1_pixart}
\end{figure}
\begin{figure}[htbp]
    \centering
    \begin{subfigure}{0.45\columnwidth}
        \centering
        \includegraphics[width=\textwidth]{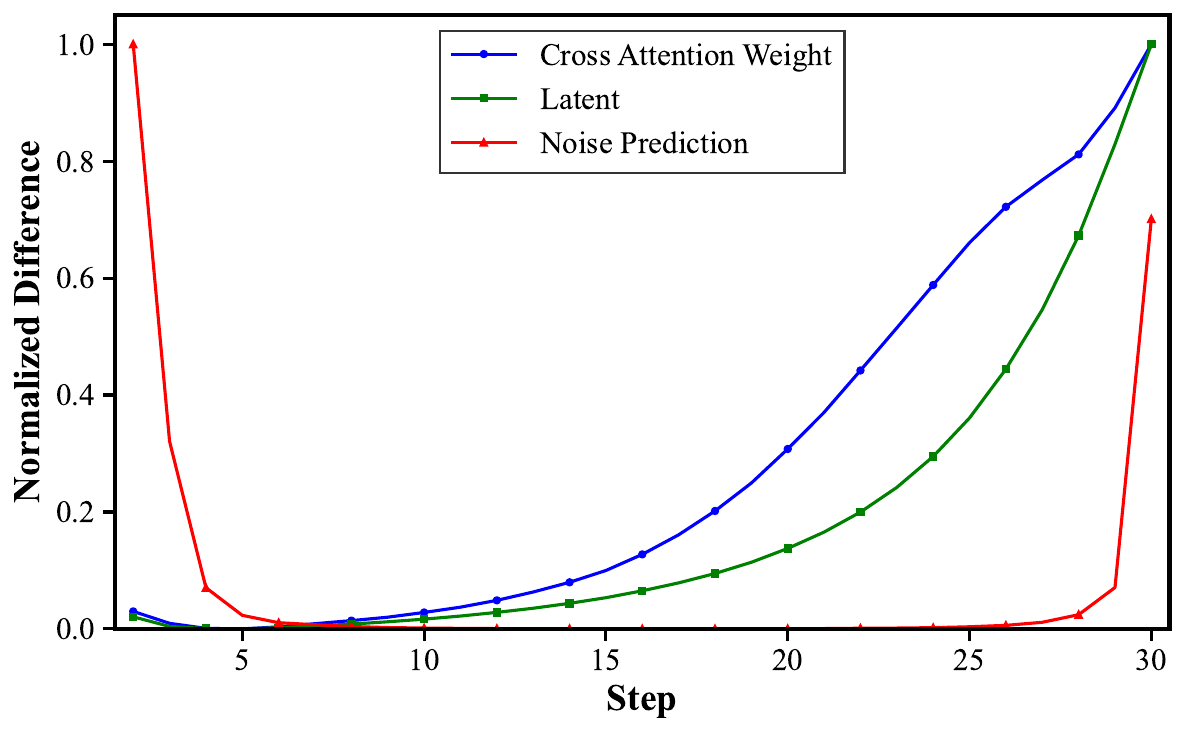}
        \caption{Averaged curves}
    \end{subfigure}
    \hfill 
    \begin{subfigure}{0.45\columnwidth}
        \centering
        \includegraphics[width=\textwidth]{fig/observation/ob1/sd_1.5_sim.pdf}
        \caption{Similarity results}
    \end{subfigure}
    \caption{Timestep-dependent evolution patterns of transition differences for the three representations of Stable Diffusion v3.5.}
    \label{ob1_sd3.5}
\end{figure}

\subsection{Observation II}\label{ob2_sub_appendix}
We further investigate how backdoor triggers affect the evolution of diffusion trajectories. We compare the representation trajectories of benign and backdoor prompts, calculating their discrepancy at each denoising timestep: $g_t^{A}=D_{\mathrm{JS}}(A_t^{b}, A_t), g_t^{z}=\mathrm{MSE}(z_t^{b}, z_t), g_t^{\epsilon}=\mathrm{MSE}(\epsilon_t^{b}, \epsilon_t)$. For different backdoor attack methods, we randomly sample 500 prompts and their corresponding trigger-injected ones. As shown in Figure~\ref{ob2}, the backdoor-related trajectories exhibit obvious discrepancies from their benign ones for the three representations. Although the temporal patterns of these discrepancies vary in different attack methods, the separation between benign and backdoor trajectories remains consistently observable over the diffusion process. To further validate Observation II, we provide additional results of different backdoor attacks. Specifically, we compare the diffusion trajectories of benign prompts with their corresponding backdoor ones in the cross-attention, latent and noise spaces. As shown in Figures~\ref{ob2_badt2i_app},~\ref{ob2_eviledit_app},~\ref{ob2_masqlora_app},~\ref{ob2_rickrolling_app} and~\ref{ob2_steba_app}, obvious discrepancies between benign and backdoor trajectories can be consistently observed in different attacks and representation spaces. We also provide quantitative similarity analysis. As can be seen, backdoor prompts of different attack methods exhibit highly similar temporal patterns of trajectory discrepancy. These results lead to our second empirical observation: \textbf{Backdoor attacks induce distinct deviations from benign diffusion trajectories.}
\begin{figure}[htbp]
    \centering
    \begin{subfigure}{0.45\columnwidth}
        \centering
        \includegraphics[width=\textwidth]{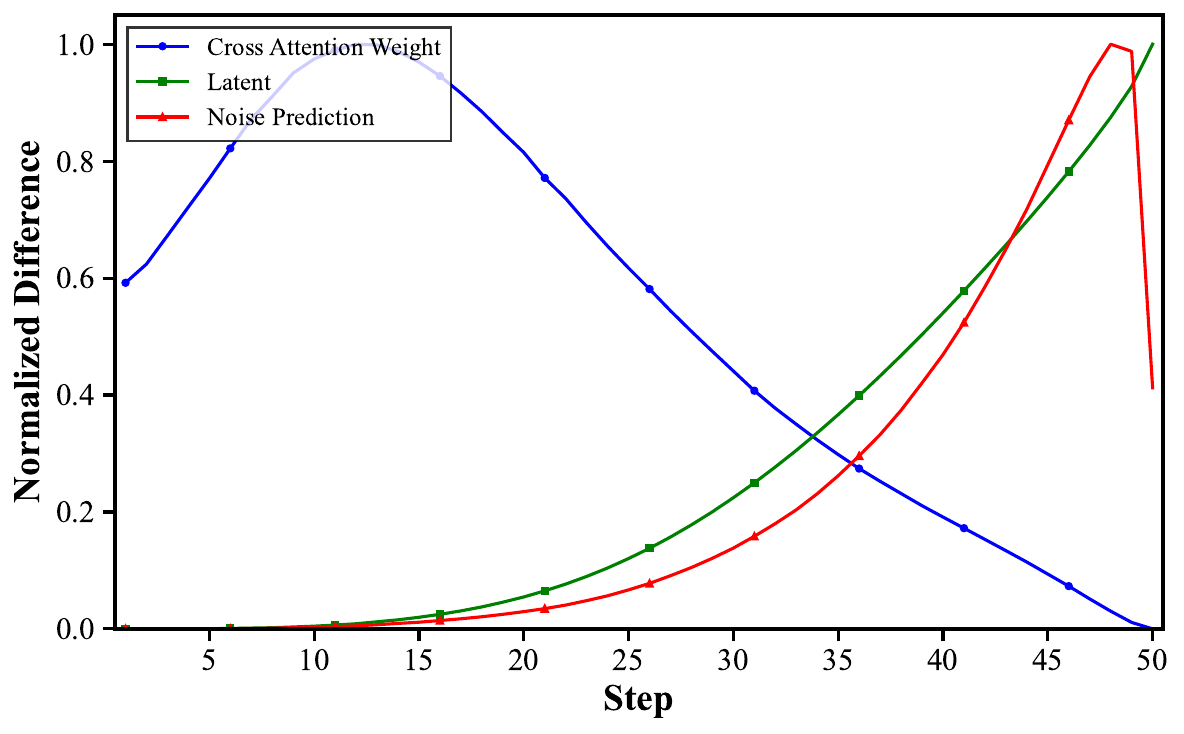}
        \caption{Averaged discrepancy curves of BadT2I (Object)}
    \end{subfigure}
    \hfill 
    \begin{subfigure}{0.45\columnwidth}
        \centering
        \includegraphics[width=\textwidth]{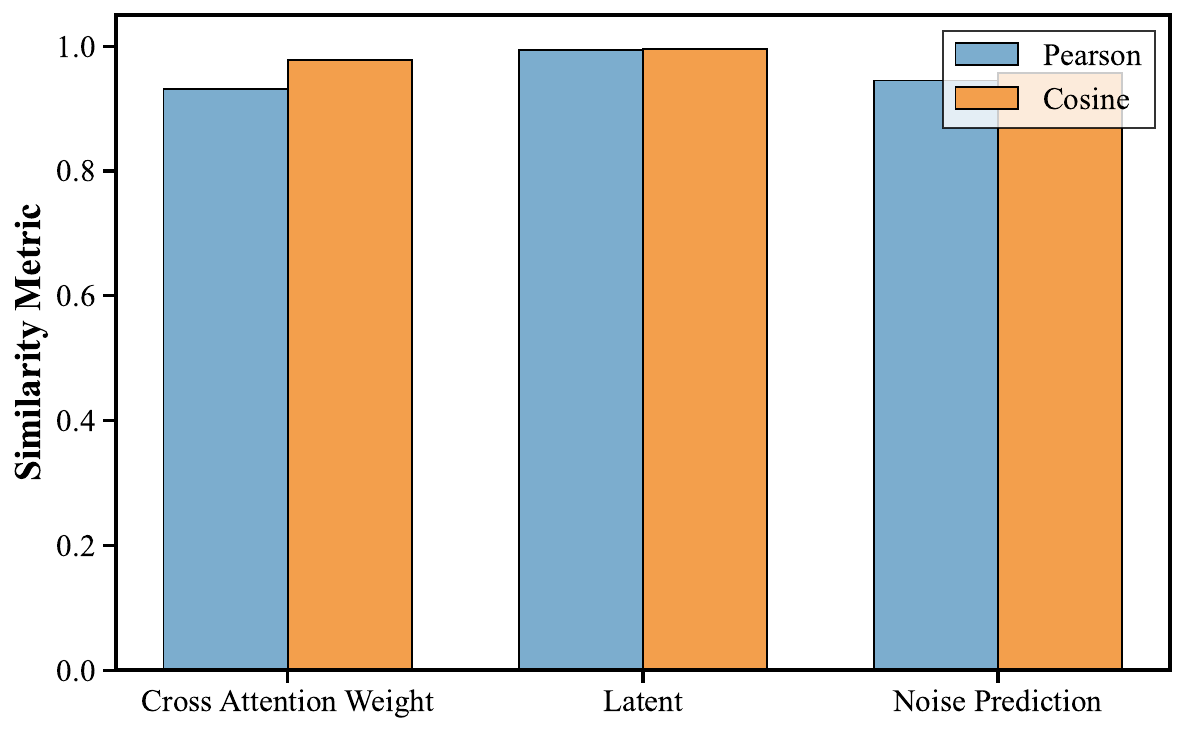}
        \caption{Similarity results of BadT2I (Object)}
    \end{subfigure}
    \begin{subfigure}{0.45\columnwidth}
        \centering
        \includegraphics[width=\textwidth]{fig/observation/ob2/badt2i_pix.pdf}
        \caption{Averaged discrepancy curves of BadT2I (Pixel)}
    \end{subfigure}
    \hfill 
    \begin{subfigure}{0.45\columnwidth}
        \centering
        \includegraphics[width=\textwidth]{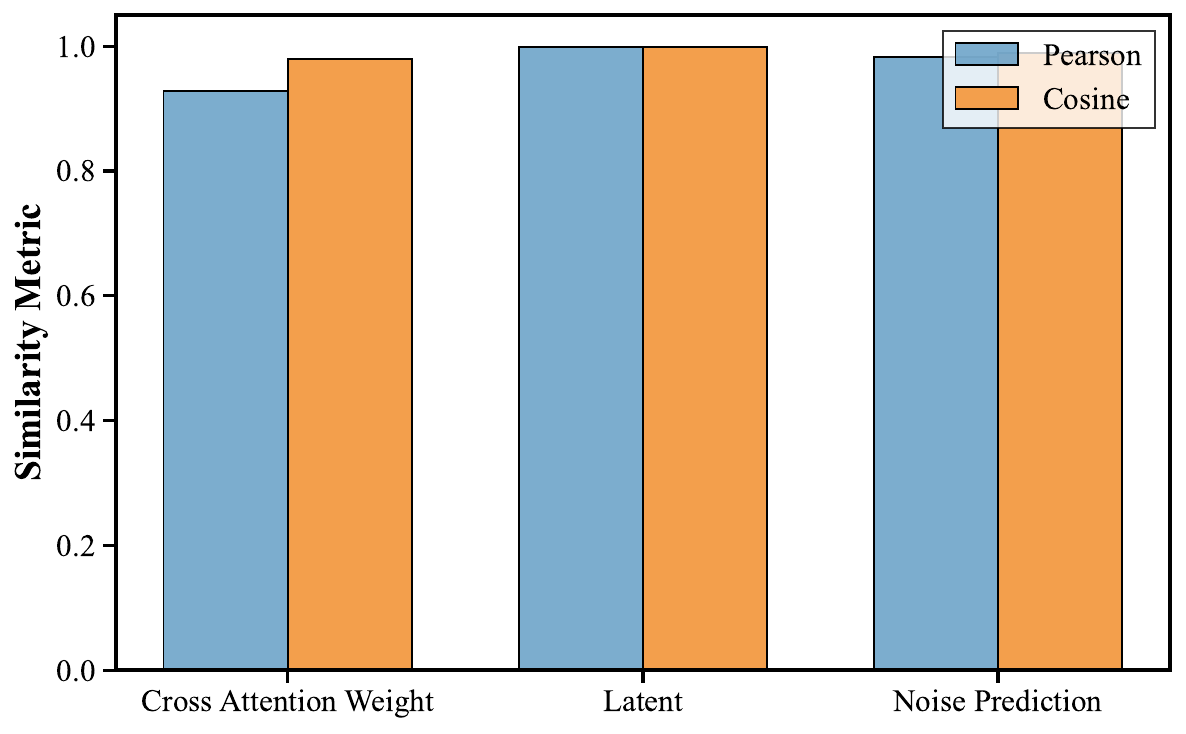}
        \caption{Similarity results of BadT2I (Pixel)}
    \end{subfigure}

    \begin{subfigure}{0.45\columnwidth}
        \centering
        \includegraphics[width=\textwidth]{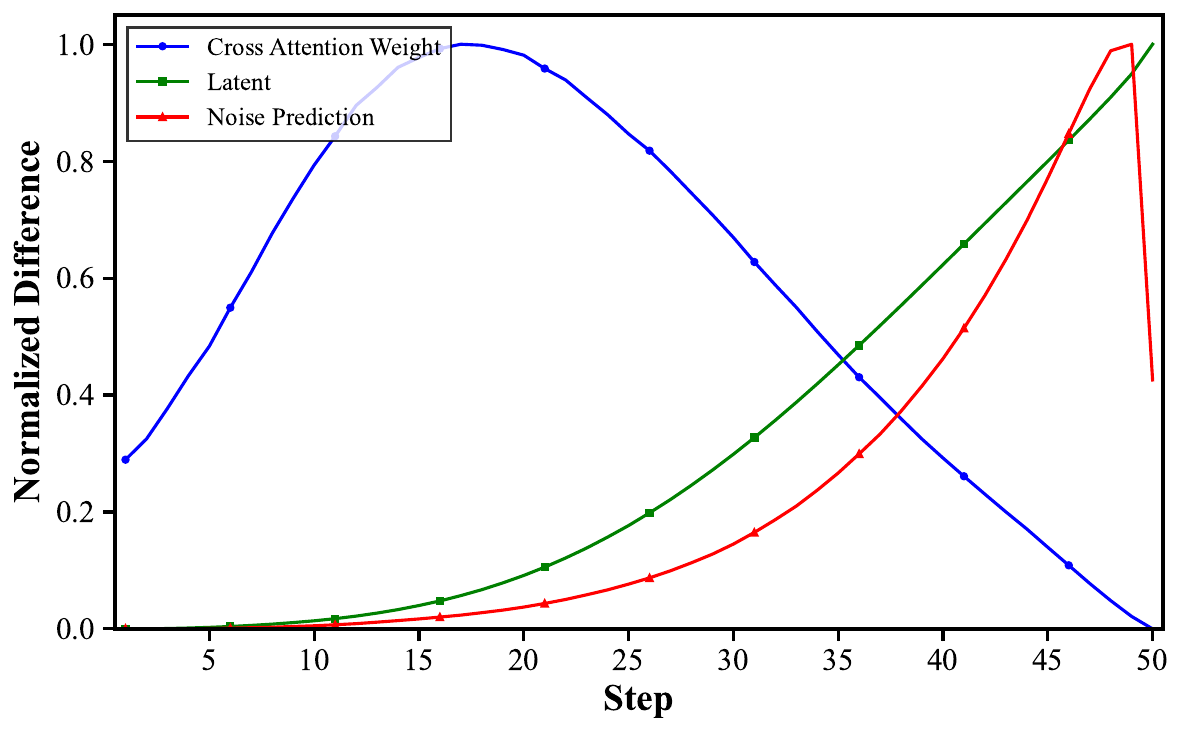}
        \caption{Averaged discrepancy curves of BadT2I (Style)}
    \end{subfigure}
    \hfill 
    \begin{subfigure}{0.45\columnwidth}
        \centering
        \includegraphics[width=\textwidth]{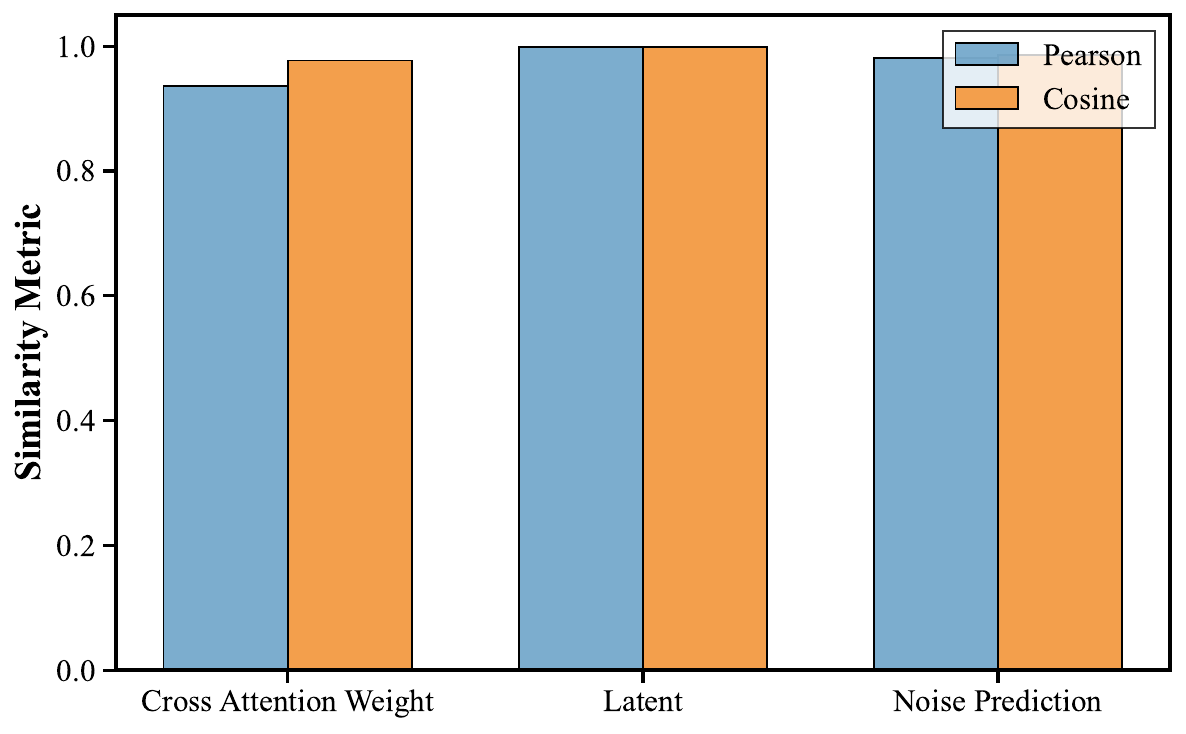}
        \caption{Similarity results of BadT2I (Style)}
    \end{subfigure}
    
    \caption{Discrepancy results in diffusion trajectories between benign and backdoor prompts of BadT2I in Stable Diffusion v1.5.}
    \label{ob2_badt2i_app}
\end{figure}

\begin{figure}[htbp]
    \centering
    \begin{subfigure}{0.45\columnwidth}
        \centering
        \includegraphics[width=\textwidth]{fig/observation/ob2/eviledit.pdf}
        \caption{Averaged discrepancy curves of EvilEdit}
    \end{subfigure}
    \hfill 
    \begin{subfigure}{0.45\columnwidth}
        \centering
        \includegraphics[width=\textwidth]{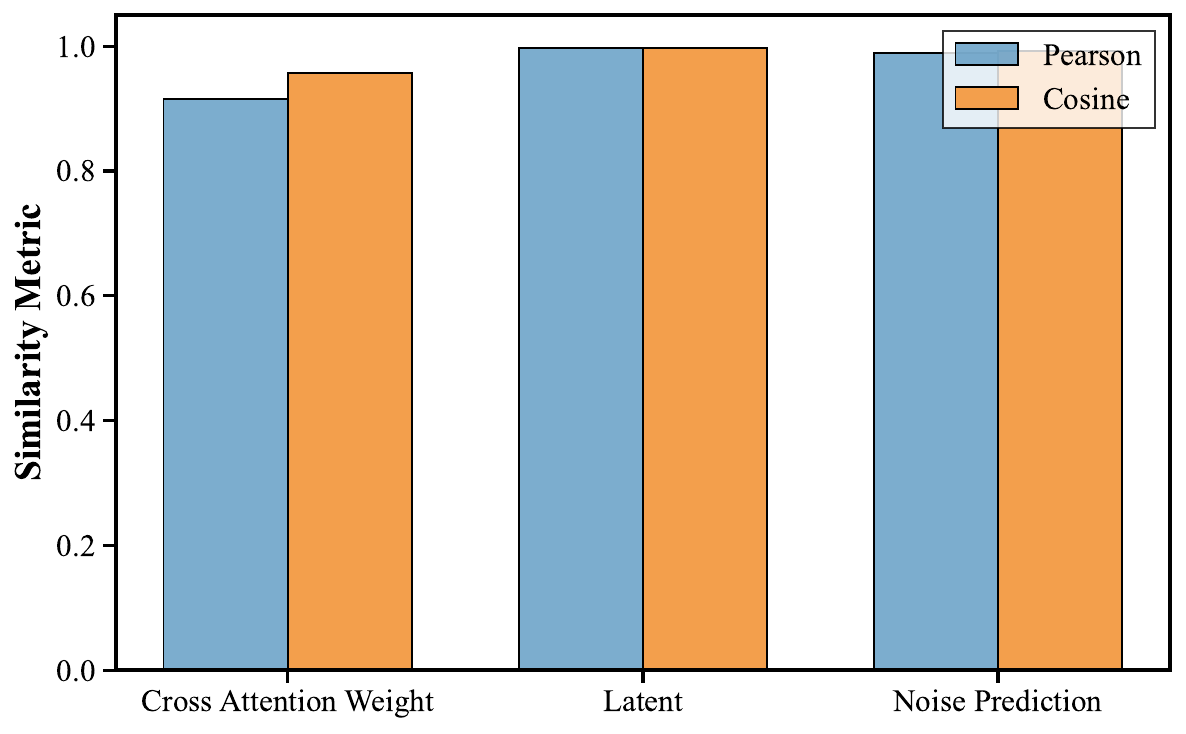}
        \caption{Similarity results of EvilEdit}
    \end{subfigure}
    
    \caption{Discrepancy results in diffusion trajectories between benign and backdoor prompts of EvilEdit in Stable Diffusion v1.5.}
    \label{ob2_eviledit_app}
\end{figure}

\begin{figure}[htbp]
    \centering
    \begin{subfigure}{0.45\columnwidth}
        \centering
        \includegraphics[width=\textwidth]{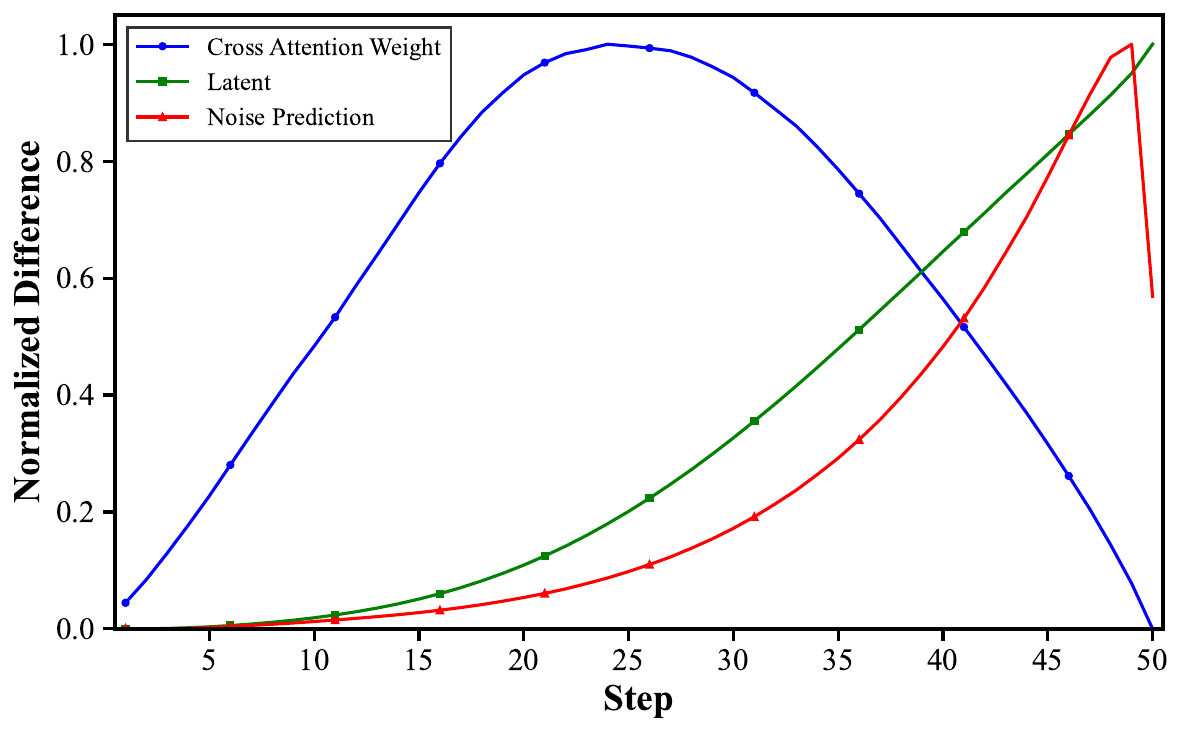}
        \caption{Averaged discrepancy curves of MasqLoRA}
    \end{subfigure}
    \hfill 
    \begin{subfigure}{0.45\columnwidth}
        \centering
        \includegraphics[width=\textwidth]{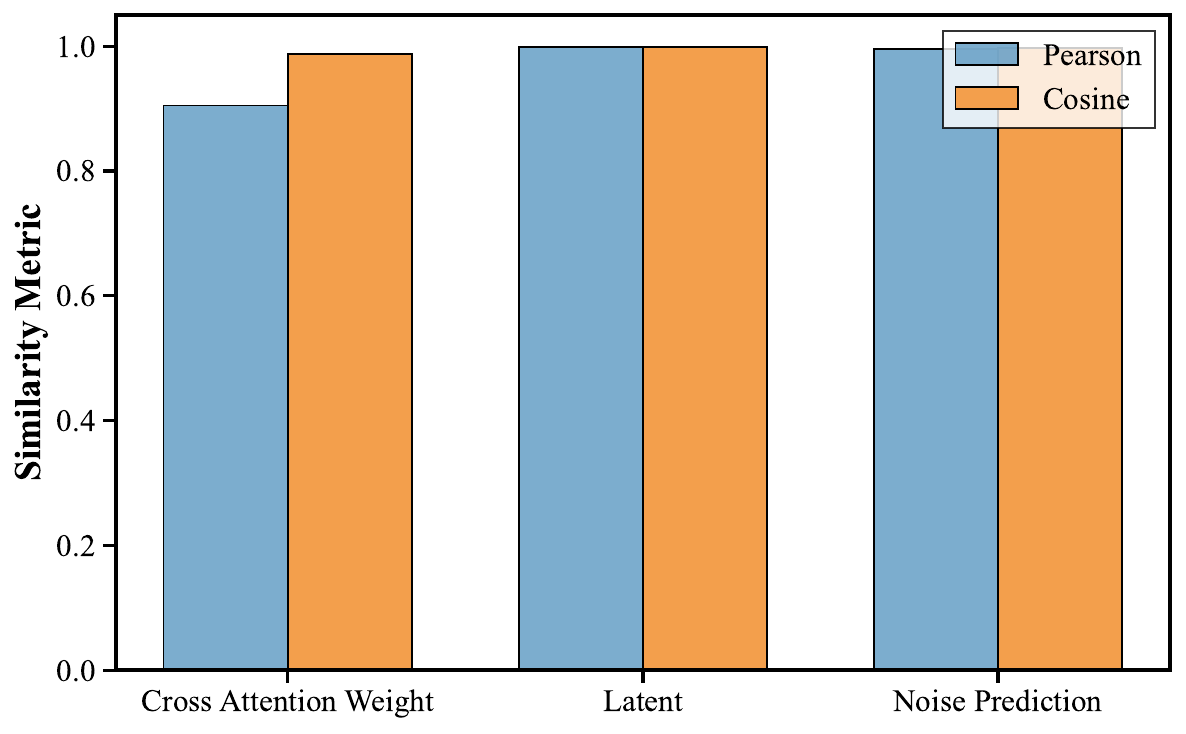}
        \caption{Similarity results of MasqLoRA}
    \end{subfigure}
    
    \caption{Discrepancy results in diffusion trajectories between benign and backdoor prompts of MasqLoRA in Stable Diffusion v1.5.}
    \label{ob2_masqlora_app}
\end{figure}

\begin{figure}[htbp]
    \centering
    \begin{subfigure}{0.45\columnwidth}
        \centering
        \includegraphics[width=\textwidth]{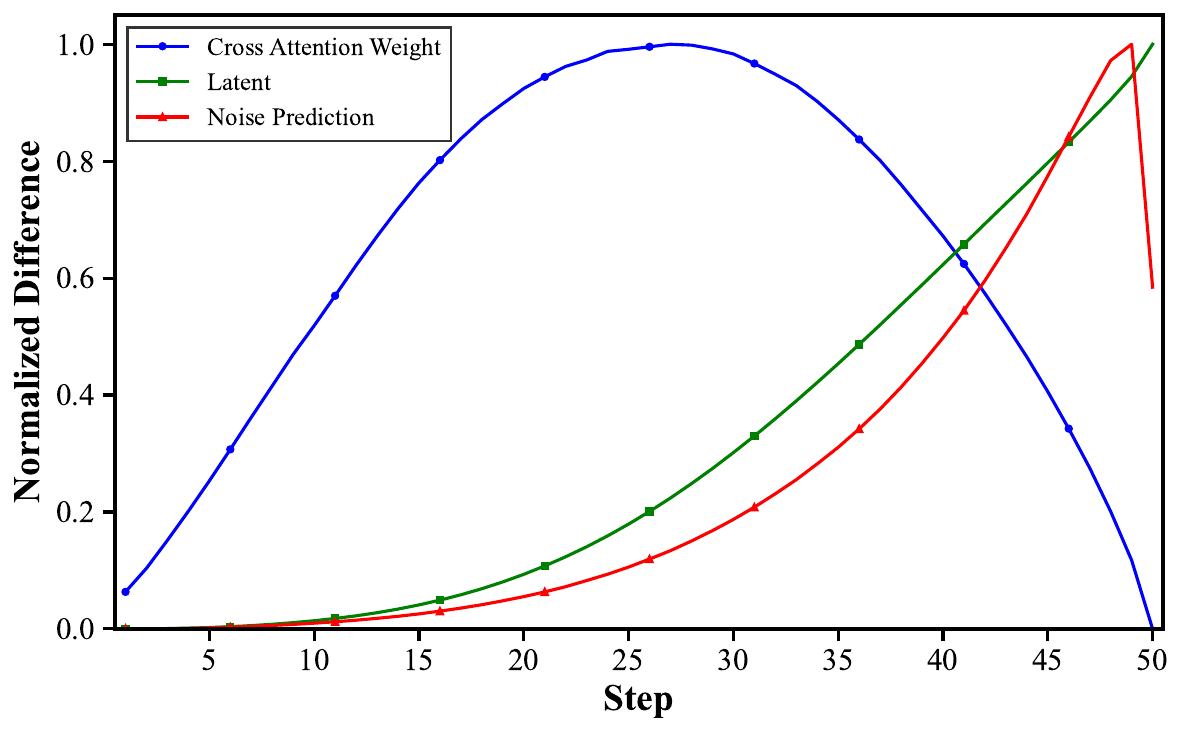}
        \caption{Averaged discrepancy curves of Rickrolling}
    \end{subfigure}
    \hfill 
    \begin{subfigure}{0.45\columnwidth}
        \centering
        \includegraphics[width=\textwidth]{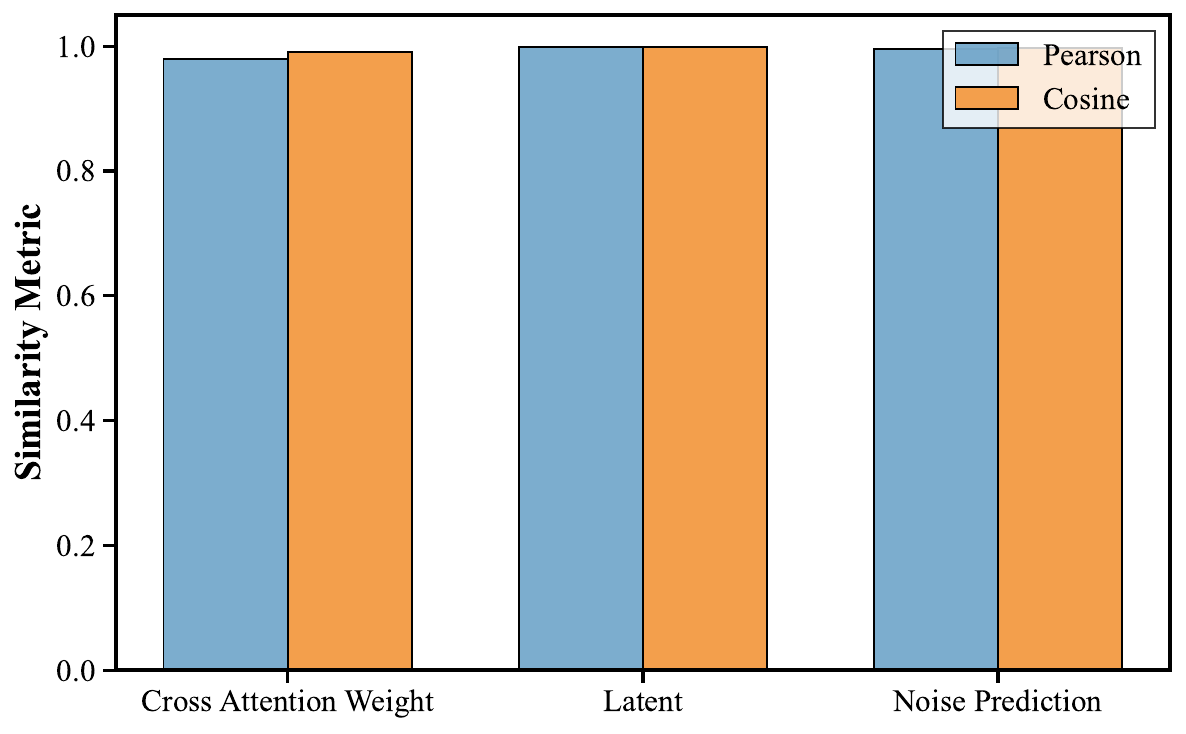}
        \caption{Similarity results of Rickrolling}
    \end{subfigure}
    
    \caption{Discrepancy results in diffusion trajectories between benign and backdoor prompts of Rickrolling in Stable Diffusion v1.5.}
    \label{ob2_rickrolling_app}
\end{figure}

\section{Transition Dynamics Analysis}
\subsection{Proof of Proposition 1}\label{proof_p1}
The proof of Proposition~\ref{proposition1} is as follows:
\begin{proof}
The benign and backdoor transition dynamics are given by:
\begin{equation}
r_{t+1}=F_t(r_t),
\qquad
r_{t+1}^{b}=F_t(r_t^{b})+\delta_t.
\end{equation}
Let $e_t=r_t^{b}-r_t$ denote the deviation between the backdoor and benign trajectories. Then, we obtain:
\begin{align}
e_{t+1}
&=r_{t+1}^{b}-r_{t+1} \nonumber\\
&=F_t(r_t^{b})-F_t(r_t)+\delta_t.
\end{align}
To use the $\ell_2$ norm and the triangle inequality, we can further obtain:
\begin{equation}
\|e_{t+1}\|_2
\leq
\|F_t(r_t^{b})-F_t(r_t)\|_2
+
\|\delta_t\|_2.
\end{equation}

\begin{figure}[t]
    \centering
    \begin{subfigure}{0.45\columnwidth}
        \centering
        \includegraphics[width=\textwidth]{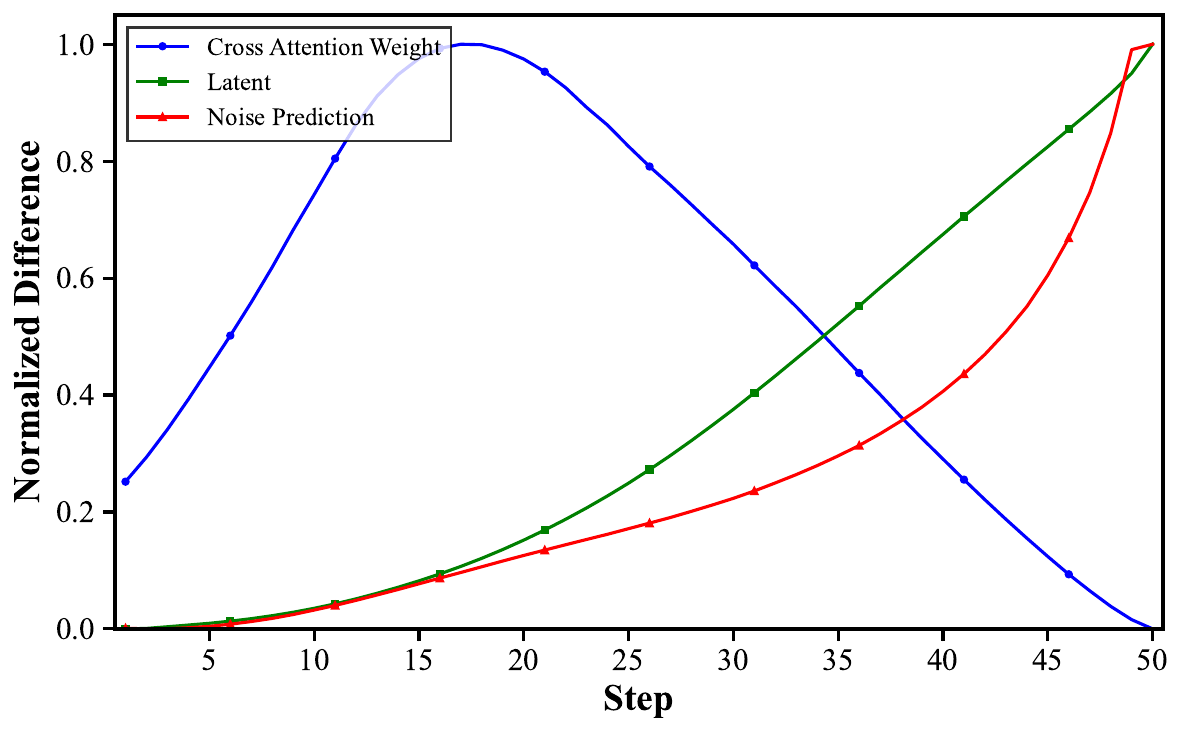}
        \caption{Averaged discrepancy curves of STEBA}
    \end{subfigure}
    \hfill 
    \begin{subfigure}{0.45\columnwidth}
        \centering
        \includegraphics[width=\textwidth]{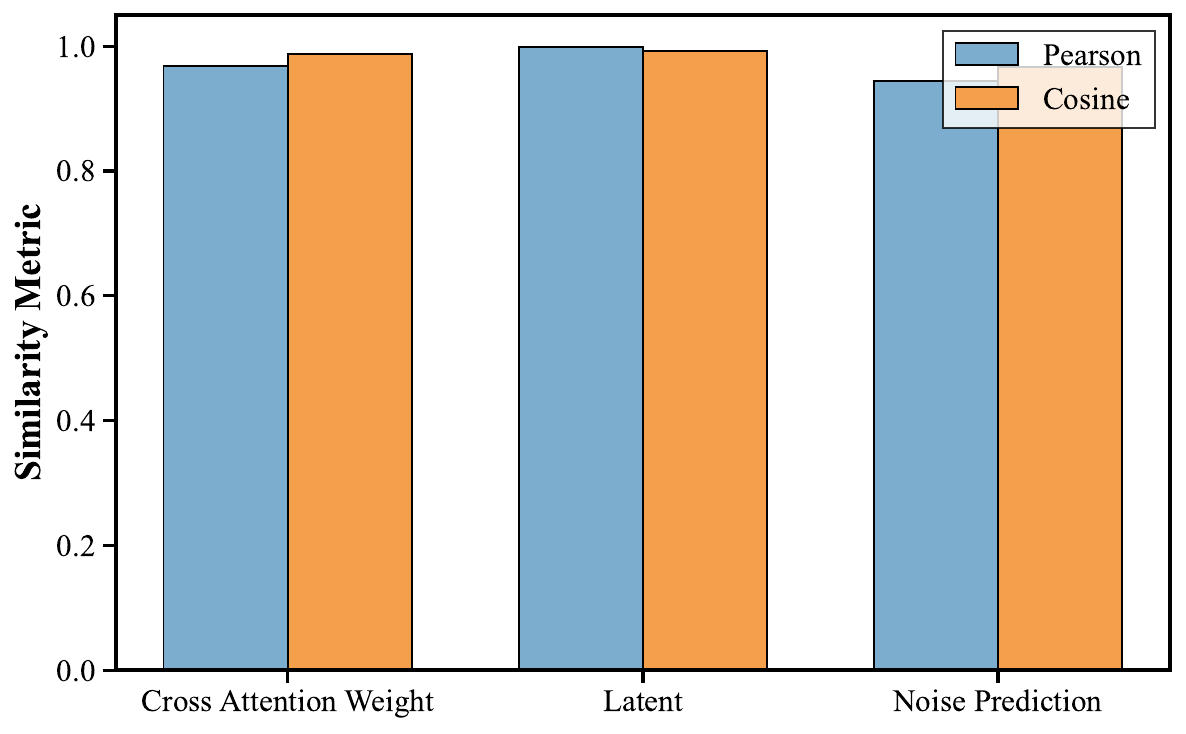}
        \caption{Similarity results of STEBA}
    \end{subfigure}
    
    \caption{Discrepancy results in diffusion trajectories between benign and backdoor prompts of STEBA in Stable Diffusion v1.5.}
    \label{ob2_steba_app}
\end{figure}

From Assumption~\ref{assumption1}, we can have:
\begin{equation}
\|F_t(r_t^{b})-F_t(r_t)\|_2
\leq
L_t\|r_t^{b}-r_t\|_2
=
L_t\|e_t\|_2. 
\end{equation}
Therefore, the following inequation holds:
\begin{equation}
\|e_{t+1}\|_2
\leq
L_t\|e_t\|_2+\|\delta_t\|_2.
\label{eq:one_step_bound}
\end{equation}

Recursively, we obtain:
\begin{align}
\|e_t\|_2
&\leq
L_{t-1}\|e_{t-1}\|_2+\|\delta_{t-1}\|_2\nonumber\\
&\leq
L_{t-1}L_{t-2}\|e_{t-2}\|_2
+
L_{t-1}\|\delta_{t-2}\|_2
+
\|\delta_{t-1}\|_2\nonumber\\
&\leq \cdots \nonumber\\
&\leq
\left(\prod_{j=0}^{t-1}L_j\right)\|e_0\|_2
+
\sum_{k=0}^{t-1}
\left(
\prod_{j=k+1}^{t-1}L_j
\right)
\|\delta_k\|_2.
\end{align}
Since the backdoor prompts includes the triggers, $e_{0} \neq 0$ holds. Thus, we have:
\begin{equation}
\|e_t\|_2 \leq \left(\prod_{j=0}^{t-1}L_j\right)\|e_0\|_2 + \sum_{k=0}^{t-1} \left( \prod_{j=k+1}^{t-1} L_j \right) \|\delta_k\|_2.
\end{equation}

\end{proof}
\subsection{Proof of Proposition 2}\label{proof_p2}
The proof of Proposition~\ref{proposition2} is as follows:
\begin{proof}
We recall that the benign and backdoor transition dynamics are defined as:
\begin{equation}
r_{t+1}=F_t(r_t),
\qquad
r_{t+1}^{b}=F_t(r_t^{b})+\delta_t.
\end{equation}
Assume that the learned predictor $F_{\theta}$ approximates the
normal transition function $F_t$ with the bounded error:
\begin{equation}
\|F_{\theta}(r,t)-F_t(r)\|_2\leq\eta_t.
\end{equation}

Given a benign trajectory, the transition prediction error is:
\begin{align}
R_t^c
&=\|r_{t+1}-F_{\theta}(r_t,t)\|_2\nonumber\\
&=\|F_t(r_t)-F_{\theta}(r_t,t)\|_2\nonumber\\
&\leq\eta_t.
\end{align}

Given a backdoor trajectory, we have:
\begin{align}
R_t^b
&=\|r_{t+1}^{b}-F_{\theta}(r_t^{b},t)\|_2\nonumber\\
&=\|F_t(r_t^{b})+\delta_t-F_{\theta}(r_t^{b},t)\|_2\nonumber\\
&=\|\delta_t+
[F_t(r_t^{b})-F_{\theta}(r_t^{b},t)]\|_2.
\end{align}
This decomposition separates the backdoor-induced perturbation from the approximation error of the learned normal dynamics predictor. Applying the reverse triangle inequality, we obtain:
\begin{align}
R_t^b
&\geq
\|\delta_t\|_2
-
\|F_t(r_t^{b})-F_{\theta}(r_t^{b},t)\|_2\nonumber\\
&\geq
\|\delta_t\|_2-\eta_t.
\end{align}
Therefore, the following two inequations can be obtained:
\begin{equation}
R_t^c\leq\eta_t,
\qquad
R_t^b\geq\|\delta_t\|_2-\eta_t.
\end{equation}
Obviously, sufficiently large backdoor-induced deviations can be distinguished from benign transitions through prediction inconsistency.
\end{proof}

\section{The Details of \textit{NDDL}}
\subsection{Calculation of Multi-space Representations}\label{ml_representations}
\subsubsection{Cross-Attention Weight}
At each diffusion timestep $t$, we denote the cross-attention weight as $A_t \in \mathbb{R}^{H\times S\times N}$, where $H$ denotes the number of attention heads, $S$ denotes the number of spatial query positions and $N$ denotes the number of tokens. Specifically, $A_{t,h,s,n}$ represents the attention weight associated with the $h$-th attention head at spatial position $s$ to the $n$-th token. Since cross-attention weights are normalized along the token dimension, we can obtain $\sum_{n=1}^{N} A_{t,h,s,n}=1$. Then, we extract $A_t$ to form four complementary descriptors: attention entropy, effective rank, token importance and head diversity. These descriptors characterize the concentration, structural complexity, token-level contribution and inter-head variation of cross-attention, respectively. It should be particularly noted that we capture joint attention weight for Stable Diffusion v3.5.

\textbf{Attention entropy:}
Attention entropy presents the concentration of token-wise attention distribution. For the $h$-th head at spatial position $s$, we calculate:
\begin{equation}
\mathcal{H}_{t,h,s}=-\sum_{n=1}^{N}
A_{t,h,s,n}\log\left(A_{t,h,s,n}\right).
\end{equation}
Then, we average the entropy over all spatial query positions:
\begin{equation}
A_{t,h}^{AE}
= \frac{1}{S}
\sum_{s=1}^{S}
\mathcal{H}_{t,h,s}.
\end{equation}
Thus, the entropy descriptor $A_t^{AE} =\left[A_{t,1}^{AE},
A_{t,2}^{AE},\dots,A_{t,H}^{AE}\right]\in\mathbb{R}^{H}$ can be formed. A larger entropy indicates that attention is distributed over a broader set of tokens, whereas a smaller entropy indicates that attention is concentrated on fewer tokens.

\textbf{Effective rank:}
Effective rank depicts the structural complexity of each attention map. For the $h$-th head, we view
$A_{t,h}\in\mathbb{R}^{S\times N}$ as a two-dimensional attention matrix and perform singular value decomposition:
\begin{equation}
A_{t,h}
=
U_{t,h}\Sigma_{t,h}V_{t,h}^{\top}. 
\end{equation}
Let $\left\{\sigma_{t,h,1},\sigma_{t,h,2},\dots,\sigma_{t,h,K} \right\} (K=\min(S,N))$ represent the singular values of $A_{t,h}$. We normalize the singular values as:
\begin{equation}
p_{t,h,k}
=
\frac{\sigma_{t,h,k}}
{\sum_{j=1}^{K}\sigma_{t,h,j}}.
\end{equation}
The entropy of the normalized singular-value distribution is:
\begin{equation}
\mathcal{H}_{t,h}^{\sigma}
=
-\sum_{k=1}^{K}
p_{t,h,k}
\log\left(p_{t,h,k}\right).
\end{equation}
Thus, we can obtain the effective rank:
\begin{equation}
A_{t,h}^{ER}
=
\exp\left(
\mathcal{H}_{t,h}^{\sigma}
\right).
\end{equation}
The descriptor of effective rank is denoted as $A_t^{ER}
=
\left[
A_{t,1}^{ER},
A_{t,2}^{ER},
\dots,
A_{t,H}^{ER}
\right]
\in\mathbb{R}^{H}$. A larger effective rank indicates greater structural complexity of the attention map, whereas a smaller value indicates that its structure is dominated by fewer components.

\textbf{Token importance:}
Token importance measures the overall attention assigned to each token. For the $n$-th token, we average its attention weights over all heads and spatial positions:
\begin{equation}
A_{t,n}^{TI}
=
\frac{1}{HS}
\sum_{h=1}^{H}
\sum_{s=1}^{S}
A_{t,h,s,n}.
\end{equation}

This descriptor is denoted as $A_t^{TI}
=
\left[
A_{t,1}^{TI},
A_{t,2}^{TI},
\dots,
A_{t,N}^{TI}
\right]
\in\mathbb{R}^{N}$, representing the average relative attention assigned to the $n$-th token at timestep $t$.

\textbf{Head diversity:}
Head diversity reflects the variation in token-level attention distribution among different heads. For each head, we first average the attention weights over all spatial positions:
\begin{equation}
\bar{A}_{t,h,n}
=
\frac{1}{S}
\sum_{s=1}^{S}
A_{t,h,s,n}.
\end{equation}
We then normalize the averaged attention weights over the token dimension:
\begin{equation}
\pi_{t,h,n}
=
\frac{\bar{A}_{t,h,n}}
{\sum_{n'=1}^{N}\bar{A}_{t,h,n'}}.
\end{equation}

For each pair of attention heads $(h_1,h_2)$, we define their mixture distribution as:
\begin{equation}
M_{t,h_1,h_2}
=
\frac{1}{2}
\left(
\pi_{t,h_1,n}
+
\pi_{t,h_2,n}
\right).
\end{equation}
JS divergence is calculated as:
\begin{equation}
D_{\mathrm{JS}}
\left(
\pi_{t,h_1,n},
\pi_{t,h_2,n}
\right)
=
\frac{1}{2}
D_{\mathrm{KL}}
\left(
\pi_{t,h_1,n}
\|
M_{t,h_1,h_2}
\right)
+
\frac{1}{2}
D_{\mathrm{KL}}
\left(
\pi_{t,h_2,n}
\|
M_{t,h_1,h_2}
\right),
\end{equation}
where KL divergence is obtained from:
\begin{equation}
D_{\mathrm{KL}}(p\|q)
=
\sum_{n=1}^{N}
p_n
\log
\frac{p_n}{q_n}.
\end{equation}

Finally, we average JS divergences over all distinct attention-head pairs:
\begin{equation}
A_t^{HD}
=
\frac{2}{H(H-1)}
\sum_{1\leq h_1<h_2\leq H}
D_{\mathrm{JS}}
\left(
\pi_{t,h_1,n},
\pi_{t,h_2,n}
\right).
\end{equation}
The head diversity descriptor is
$
A_t^{HD}
\in\mathbb{R}$. A larger value indicates greater variation in token-level attention distribution among different heads, whereas a smaller value indicates more consistent attention behaviors.
\subsubsection{Latent}
At each diffusion timestep $t$, the latent state is denoted as $z_t \in \mathbb{R}^{C\times H\times W}$, where $C$ denotes the number of latent channels, $H$ and $W$ denote the spatial resolution. We construct the compact representation including four descriptors: channel norm, trajectory curvature, frequency energy and temporal variation. These descriptors describe the magnitude, variation degree, spectral structure and local transition behavior of the latent trajectory. Notably, for the boundary steps where the preceding states are unavailable, we apply zero-padding.

\textbf{Channel norm:} Latent norm reflects the overall magnitude of each latent channel. For the $c$-th latent channel $z_{t,c}\in\mathbb{R}^{H\times W}$, we obtain:
\begin{equation}
z_{t,c}^{CN}
=
\left\|z_{t,c}\right\|_2
=
\sqrt{
\sum_{i=1}^{H}
\sum_{j=1}^{W}
z_{t,c,i,j}^{2}
}.
\end{equation}
The descriptor of latent norm is
$z_t^{CN}
=
\left[
z_{t,1}^{CN},
z_{t,2}^{CN},
\dots,
z_{t,C}^{CN}
\right]
\in\mathbb{R}^{C}$. A larger latent norm indicates a larger overall magnitude of the corresponding latent channel, providing a compact representation of the instantaneous latent state.

\textbf{Trajectory curvature:} Latent curvature is calculated from the second-order temporal variation of the latent trajectory. We first obtain the consecutive transition difference of the $c$-th latent channel as:
\begin{equation}
\Delta z_{t,c}
=
z_{t,c}-z_{t-1,c}.
\end{equation}
The curvature is then calculated as the variation between two consecutive transition differences:
\begin{equation}
z_{t,c}^{TC}
=
\left\|
\Delta z_{t,c}
-
\Delta z_{t-1,c}
\right\|_2.
\end{equation}
Equivalently, we have:
\begin{equation}
z_{t,c}^{TC}
=
\left\|
z_{t,c}
-
2z_{t-1,c}
+
z_{t-2,c}
\right\|_2.
\end{equation}
Thus, the descriptor of latent curvature is
$z_t^{TC}
=
\left[
z_{t,1}^{TC},
z_{t,2}^{TC},
\dots,
z_{t,C}^{TC}
\right]
\in\mathbb{R}^{C}$. A larger curvature indicates a stronger change in the local transition of the latent trajectory, while a smaller value suggests a more consistent evolution between consecutive denoising steps.

\textbf{Frequency-domain energy:}
Latent frequency energy reflects the spectral structure of each latent channel. For the $c$-th latent channel, we implement a two-dimensional Fourier transform:
\begin{equation}
\mathcal{F}_{t,c}(u,v)
=
\operatorname{FFT}\left(z_{t,c}\right),
\end{equation}
where $(u,v)$ denotes the coordinate. The corresponding power spectrum is defined as:
\begin{equation}
P_{t,c}(u,v)
=
\left|
\mathcal{F}_{t,c}(u,v)
\right|^{2}.
\end{equation}
We divide the frequency domain into $B$ frequency regions
$\{\Omega_1,\Omega_2,\dots,\Omega_B\}$ and compute the energy for each region:
\begin{equation}
E_{t,c,b}
=
\sum_{(u,v)\in\Omega_b}
P_{t,c}(u,v),
\qquad
b=1,\dots,B.
\end{equation}
To reduce the dynamic range of frequency energy, we exploit a logarithmic transformation:
\begin{equation}
z_{t,c,b}^{FE}
=
\log\left(1+E_{t,c,b}\right).
\end{equation}
Thus, the descriptor of frequency energy is
$
z_t^{FE}
=
\left[
z_{t,1,1}^{FE},
\dots,
z_{t,1,B}^{FE},
\dots,
z_{t,C,1}^{FE},
\dots,
z_{t,C,B}^{FE}
\right]
\in\mathbb{R}^{CB}.$
We set $B=3$ corresponding to low-frequency, middle-frequency and high-frequency regions. This descriptor reveals how latent representation is distributed over different frequency components during denoising.

\textbf{Temporal variation:}
Latent temporal variation reflects the magnitude of the local transition between two consecutive denoising steps. For the $c$-th latent channel, we compute:
\begin{equation}
z_{t,c}^{TV}
=
\left\|
z_{t,c}
-
z_{t-1,c}
\right\|_2.
\end{equation}
The descriptor of temporal variation is
$z_t^{TV}
=
\left[
z_{t,1}^{TV},
z_{t,2}^{TV},
\dots,
z_{t,C}^{TV}
\right]
\in\mathbb{R}^{C}$. A larger norm indicates a stronger transition between consecutive latent states, whereas a smaller value represents relatively mild local evolution.
\subsubsection{Noise}
At each diffusion timestep $t$, we denote the predicted noise as
$\epsilon_t \in \mathbb{R}^{C\times H\times W}$,
where $C$ denotes the number of noise channels, $H$ and $W$ denote the spatial resolution. We obtain the descriptor of noise predictions includes channel norm, channel variance, frequency energy and temporal variation. Similarly, we utilize zero-padding for the unavailable preceding states of the boundary steps. 

\textbf{Channel norm:}
Noise channel norm calculates the overall magnitude of each noise channel. For the $c$-th channel $\epsilon_{t,c} \in \mathbb{R}^{H \times W}$, we have: 
\begin{equation} \epsilon_{t,c}^{CN} = \left\|\epsilon_{t,c}\right\|_2 = \sqrt{ \sum_{i=1}^{H} \sum_{j=1}^{W} \epsilon_{t,c,i,j}^{2} }. 
\end{equation}
The descriptor of channel norm is defined as $\epsilon_t^{CN} = \left[ \epsilon_{t,1}^{CN}, \epsilon_{t,2}^{CN}, \ldots, \epsilon_{t,C}^{CN} \right] \in \mathbb{R}^{C}$.

\textbf{Channel variance:}
Noise channel variance represents the spatial dispersion of the predicted values of each noise channel. 
Given the $c$-th predicted-noise channel $\epsilon_{t,c}\in\mathbb{R}^{H\times W}$, we first compute its spatial mean as:
\begin{equation}
\mu_{t,c}^{\epsilon}
=
\frac{1}{HW}
\sum_{i=1}^{H}
\sum_{j=1}^{W}
\epsilon_{t,c,i,j}.
\end{equation}

The variance of is then obtained:
\begin{equation}
\epsilon_{t,c}^{CV}
=
\frac{1}{HW}
\sum_{i=1}^{H}
\sum_{j=1}^{W}
\left(
\epsilon_{t,c,i,j}
-
\mu_{t,c}^{\epsilon}
\right)^2.
\end{equation}

The descriptor of channel variance at timestep $t$ is denoted as
$\epsilon_t^{CV}
=
\left[
\epsilon_{t,1}^{CV},
\epsilon_{t,2}^{CV},
\ldots,
\epsilon_{t,C}^{CV}
\right]
\in \mathbb{R}^{C}$.
A larger variance indicates a more dispersed spatial distribution of the noise values, whereas a smaller variance indicates that the values are more concentrated around the channel mean. 

\textbf{Frequency-domain energy:}
Noise frequency energy manifests the spectral distribution of each noise channel. 
For the $c$-th channel $\epsilon_{t,c}\in\mathbb{R}^{H\times W}$, we transform the spatial representation into the frequency domain using the two-dimensional discrete Fourier transform:
\begin{equation}
\mathcal{F}_{t,c}^{\epsilon}(u,v)
=
\operatorname{FFT}\left(\epsilon_{t,c}\right),
\end{equation} 

The corresponding power spectrum is then computed as:
\begin{equation}
P_{t,c}^{\epsilon}(u,v)
=
\left|
\mathcal{F}_{t,c}^{\epsilon}(u,v)
\right|^2.
\end{equation}

The energy of the $c$-th channel for each region is obtained:
\begin{equation}
E_{t,c,b}^{\epsilon}
=
\sum_{(u,v)\in\Omega_b}
P_{t,c}^{\epsilon}(u,v),
\qquad
b=1,\dots,B.
\end{equation}

Also, we further apply a logarithmic transformation:
\begin{equation}
\epsilon_{t,c,b}^{FE}
=
\log\left(
1+E_{t,c,b}^{\epsilon}
\right).
\end{equation}

The descriptor of frequency energy is
$\epsilon_t^{FE}
=
\left[
\epsilon_{t,1,1}^{FE},
\dots,
\epsilon_{t,1,B}^{FE},
\dots,
\epsilon_{t,C,1}^{FE},
\dots,
\epsilon_{t,C,B}^{FE}
\right]
\in\mathbb{R}^{CB}$.

\textbf{Temporal variation:}
Noise temporal variation reflects the magnitude of the local transition between two consecutive noise predictions. For the $c$-th noise channel, we compute:
\begin{equation}
\epsilon_{t,c}^{TV}
=
\left\|
\epsilon_{t,c}
-
\epsilon_{t-1,c}
\right\|_2.
\end{equation}
The descriptor of temporal variation is
$\epsilon_t^{TV}
=
\left[
\epsilon_{t,1}^{TV},
\epsilon_{t,2}^{TV},
\dots,
\epsilon_{t,C}^{TV}
\right]
\in\mathbb{R}^{C}$. 
\begin{table}[htbp]
  \centering
  \caption{The descriptors of the trajectory representation mapping in \textit{NDDL}.}
  \label{representation_descriptors}
  \resizebox{\textwidth}{!}{%
    \begin{tabular}{m{3cm}m{4cm}m{6cm}}
      \toprule
      \textbf{Representation} & \textbf{Descriptor} & \textbf{Characterized Property} \\
      \midrule
      \multirow{4}{=}{Cross-Attention Weight}
        & Attention Entropy      & Distribution concentration      \\
        & Effective Rank         & Structural complexity           \\
        & Token Importance       & Token-level contribution        \\
        & Head Diversity         & Inter-head variation            \\
      \midrule
      \multirow{4}{*}{Latent}
        & Channel Norm           & State magnitude                 \\
        & Curvature              & Second-order temporal variation \\
        & Frequency Energy       & Spectral structure              \\
        & Temporal Variation   & First-order temporal variation  \\
      \midrule
      \multirow{4}{*}{Noise Prediction}
        & Channel Norm           & Prediction magnitude            \\
        & Channel Variance       & Spatial dispersion              \\
        & Frequency Energy       & Spectral structure              \\
        & Temporal Variation     & First-order temporal variation  \\
      \bottomrule
    \end{tabular}%
  }
\end{table}

Table~\ref{representation_descriptors} provides the descriptors of the trajectory representation mapping. We do not claim that these descriptors are exhaustive or uniquely optimal. We aim to demonstrate that learning normal transition dynamics in a compact multi-space representation provides an effective strategy for backdoor defense of T2I diffusion models.

\subsection{Normalization Details}\label{Normalization}
We first normalize each feature dimension independently using the statistics computed exclusively from the benign training trajectories. Let $x_{i,k}$ denote the value of the $k$-th feature dimension in the $i$-th benign training sample. For each feature dimension $k$, we compute its median as $m_k$ and the median absolute deviation (MAD) as $ \operatorname{MAD}_k $. The normalized feature is then calculated as:
\begin{equation} 
\hat{x}_{i,k} = \frac{ x_{i,k}-m_k }{ 1.4826\, \max\left(\operatorname{MAD}_k,\varepsilon\right) }, 
\end{equation} 
where $\varepsilon=10^{-8}$ prevents extreme case when the MAD approaches zero. The constant $1.4826$ rescales the MAD to provide a robust estimate comparable to the standard deviation under a Gaussian distribution.

After feature-wise normalization, the obtained descriptors are grouped into three representation blocks: cross-attention, latent and noise prediction. The normalized feature block for representation is represented as: 
\begin{equation} 
\hat{r}_t^{m} \in \mathbb{R}^{d_m}, \qquad m\in\{A,z,\epsilon\},  
\end{equation} 
where $d_m$ is its dimensionality. 

To mitigate the dimensionality-induced imbalance, we then scale each block by the square root of its dimensionality: 
\begin{equation} 
\bar{r}_t^{m} = \frac{ \hat{r}_t^{m} }{ \sqrt{d_m} }, \qquad m\in\{A,z,\epsilon\}. 
\end{equation}

\subsection{The architecture of normal dynamics network}\label{network_appendix}
The normal dynamics network is a residual MLP conditioned on timestep $t$. Given the current state $r_{t}$, the network predicts the state increment $\Delta \hat{r}_{t}$ between adjacent steps. Specially, the timestep $t$ is first mapped to a 32-dimensional vector via sinusoidal positional encoding, and then processed by a two-layer MLP to obtain the time representation $e_{t}$. The state $r_{t}$ and time representation $e_{t}$ are concatenated and projected into a 512-dimensional hidden space via a linear layer followed by LayerNorm and GELU. Then, the hidden output is input into 4 residual blocks. Finally, the output head maps the hidden result to the state space, i.e., the increment $\Delta \hat{r}_{t}$. 
The details of normal dynamics network can be seen in Tables~\ref{tab:ndn} and~\ref{tab:resblock}. 

\begin{table}[h]
\centering
\caption{The architecture of normal dynamics network, where $B$ is batch size and $d$ is the dimension of the compact descriptors.}
\label{tab:ndn}
\small
\begin{tabular}{lll}
\toprule
Stage & Operation & Output shape \\
\midrule
State input    & $r_t$                                          & $B\times d$ \\
Step input      & $t$                                            & $B\times1$   \\
Time embedding  & $\sin/\cos$ $\to$ MLP                          & $B\times32$  \\
Concatenation   & $[r_t;\,e_t]$                                  & $B\times (d + 32)$ \\
Input projection& Linear $\to$ LayerNorm $\to$ GELU              & $B\times512$ \\
Dynamics core   & ResidualBlock $\times 4$                       & $B\times512$ \\
Output head     & Linear $\to$ GELU $\to$ Linear                 & $B\times d$ \\
\midrule
Output          & $\Delta\hat{r}_t$                               & $B\times d$ \\
\bottomrule
\end{tabular}
\end{table}

\begin{table}[h]
\centering
\caption{Residual block utilized in the normal dynamics network.}
\label{tab:resblock}
\small
\begin{tabular}{lll}
\toprule
Sub-layer & Operation & Shape \\
\midrule
Input         & $x$                       & $B\times512$  \\
Linear (up)   & Linear $\to$ GELU         & $B\times1024$ \\
Dropout       & Dropout       & $B\times1024$ \\
Linear (down) & Linear                    & $B\times512$  \\
Normalize     & LayerNorm                  & $B\times512$  \\
Add (skip)    & $x + \text{Normalize}$     & $B\times512$  \\
\bottomrule
\end{tabular}
\end{table}
\subsection{Details of Low-Semantic Words}\label{appendix_neutral_tokens}
For trigger localization, we construct an initial substitution set using common function words with low semantics, such as \emph{a}, \emph{an}, \emph{the}, \emph{this}, \emph{that}, \emph{some}, \emph{and}, \emph{with} and \emph{of}. These words typically introduce less semantic perturbations.

\subsection{Pseudocode of \textit{NDDL}}\label{pseudocode}
Algorithm~\ref{alg:training} presents the training process of normal dynamics model. Algorithms~\ref{alg:detection} and~\ref{alg:localization} show the details of defense process, including both backdoor detection and trigger localization.
\begin{algorithm}[h]
\caption{Normal Diffusion Dynamics Learning}
\label{alg:training}
\begin{algorithmic}[1]
\Require Benign prompt set $\mathcal{P}_{b}$; target diffusion model $\mathcal{M}$
\Ensure Normal dynamics model $F_{\theta}$

\For{each benign prompt $p \in \mathcal{P}_{b}$}
    \State Obtain cross-attention weight
    $A_t$, latent $z_t$ and noise $\epsilon_t$ from $\mathcal{M}$ with prompt $p$
    \For{each denoising timestep $t$}
        \State Extract the compact representation
        $r_t \gets \phi(A_t,z_t,\epsilon_t)$
    \EndFor
\EndFor

\State Estimate feature-wise normalization statistics from benign trajectories
\State Normalize all trajectory representations $r_t$

\For{each benign transition $(r_t,r_{t+1})$}
    \State Predict next-step representation
    $\hat r_{t+1} \gets F_{\theta}(r_t,t)$
    \State Compute dynamics prediction loss
    $\mathcal{L}$ in~\eqref{model_loss}
    \State Update $\theta$ by minimizing $\mathcal{L}$
\EndFor
\State \Return $F_{\theta}$
\end{algorithmic}
\end{algorithm}

\section{Experimental results}

\subsection{Detection Results}\label{detection_appendix}

Table~\ref{detection_results_sdxl} summarizes the backdoor detection results on Stable Diffusion XL. Remarkably, \textit{NDDL} consistently achieves superior performance, exceeding 97\% in both ACC and AUROC across all evaluated attacks. The most significant improvement is observed on STEBA, where \textit{NDDL} surpasses the strongest baseline by a substantial margin, increasing ACC from 83.3\% to 97.2\% and AUROC from 83.1\% to 97.0\%. These results further corroborate the effectiveness and generalizability of \textit{NDDL} across diverse diffusion architectures.

\begin{algorithm}[h]
\caption{Backdoor Detection}
\label{alg:detection}
\begin{algorithmic}[1]
\Require Test prompt $p$; target diffusion model $\mathcal{M}$;
normal dynamics model $F_{\theta}$;
selected timestep interval $\mathcal{T}$;
window length $K$;
detection threshold $\lambda_1$
\Ensure Detection result and anomaly score $S(p)$

\State Run $\mathcal{M}$ with $p$ and collect
$\{A_t,z_t,\epsilon_t\}_{t\in\mathcal{T}}$

\For{each timestep $t \in \mathcal{T}$}
    \State Extract and normalize
    $r_t \gets \phi(A_t,z_t,\epsilon_t)$
\EndFor

\For{each transition $t \in \mathcal{T}$}
    \State $\hat r_{t+1} \gets F_{\theta}(r_t,t)$
    \State $E_t \gets \frac{1}{d} \| r_{t+1} - \hat{r}_{t+1} \|_2^2$
\EndFor

\State Partition $\mathcal{T}$ into equal-length temporal windows
$\mathcal{W}=\{W_1,\ldots,W_n\}$

\For{each window $W_i \in \mathcal{W}$}
    \State $S_i(p) \gets \frac{1}{|W_i|} \sum_{t \in W_i} E_t.$
\EndFor

\State   Obtain the final score $S(p) = \max_{i=1,\dots,n} S_i(p)$

\If{$S(p) > \lambda_1$}
    \State $y \gets \textsc{Backdoor}$
\Else
    \State $y \gets \textsc{Benign}$
\EndIf

\State \Return $y$, $S(p)$
\end{algorithmic}
\end{algorithm}

\begin{algorithm}[h]
\caption{Trigger Localization}
\label{alg:localization}
\begin{algorithmic}[1]
\Require Suspicious prompt $p=(w_1,\ldots,w_m)$;
low-semantic substitute set $\mathcal{V}_c$;
localization threshold $\lambda_2$
\Ensure Localized trigger token set $\mathcal{G}$

\State Compute original anomaly score $S_0 \gets S(p)$
\State $\mathcal{G} \gets \emptyset$

\For{each candidate token $w_i$ in $p$}
    \For{each substitute $v \in \mathcal{V}_c$}
        \State Construct substituted prompt $p^{(i\rightarrow v)}$
        \State Compute $S_i^v \gets S(p^{(i\rightarrow v)})$
    \EndFor

    \State $C_i \gets
    S_0 -
    \operatorname{Median}_{v\in\mathcal{V}_c} S_i^v$

    \If{$C_i > \lambda_{2}$}
        \State $\mathcal{G} \gets \mathcal{G}\cup\{w_i\}$
    \EndIf
\EndFor

\State \Return $\mathcal{G}$
\end{algorithmic}
\end{algorithm}

\subsection{Localization Results}\label{localization_appendix}

As presented in Table~\ref{localization_results_sdxl}, \textit{NDDL} also attains the best trigger-localization performance on Stable Diffusion XL. In addition to one-token and multi-token triggers, \textit{NDDL} further yields substantial improvements in localizing both special-character and sentence-level triggers. These results further underscore the effectiveness and generalizability of \textit{NDDL} in localizing diverse trigger forms across different diffusion architectures.

\begin{table}[htbp]
  \centering
  \caption{Evaluation results of different detection methods against various attacks on Stable Diffusion XL. Bold indicates the best performance, and underlined denotes the second best.}
  \label{detection_results_sdxl}
  \resizebox{\textwidth}{!}{%
    \begin{tabular}{lcccccccccc} 
      \toprule
      \multirow{2}{*}{\textbf{Method}} & \multicolumn{2}{c}{\textbf{BadT2I}} & \multicolumn{2}{c}{\textbf{EvilEdit}} & \multicolumn{2}{c}{\textbf{MasqLoRA}} & \multicolumn{2}{c}{\textbf{Rickrolling}} & \multicolumn{2}{c}{\textbf{STEBA}} \\
      \cmidrule(lr){2-3} \cmidrule(lr){4-5} \cmidrule(lr){6-7} \cmidrule(lr){8-9} \cmidrule(lr){10-11}
       & ACC $\uparrow$& AUROC $\uparrow$ & ACC $\uparrow$ & AUROC $\uparrow$ & ACC $\uparrow$ & AUROC $\uparrow$ & ACC $\uparrow$ & AUROC $\uparrow$ & ACC $\uparrow$ & AUROC $\uparrow$ \\
      \midrule
      UFID      & 64.8 & 65.7 & 65.3 & 67.0 & 65.8 & 66.6 & 55.8 & 56.5 & 52.8 & 53.3 \\
      T2IShield & 84.3 & 84.9 & 83.5 & 84.5 & 86.8 & 87.4 & 81.3 & 81.8 & 75.8 & 76.9 \\
      NaviT2I   & 93.2 & 92.9 & 96.2 & 96.4 & 90.7 & 90.9 & 84.1 & 84.9 & 70.8 & 71.6 \\
      STEDF     & \underline{98.2} & \underline{98.3} & \textbf{98.8} & \textbf{99.1} & \underline{96.3} & \underline{96.1} & \underline{96.5} & \underline{96.9} & \underline{83.3} & \underline{83.1} \\
      \textbf{\textit{NDDL} (Ours)} & \textbf{98.7} & \textbf{99.0} & \underline{98.5} & \underline{98.7} & \textbf{98.1} & \textbf{98.2} & \textbf{97.3} & \textbf{97.1} & \textbf{97.2} & \textbf{97.0} \\
      \bottomrule
    \end{tabular}%
  }
\end{table}

\begin{table}[htbp]
  \centering
  \caption{Evaluation results of trigger localization using different methods on Stable Diffusion
XL.}
  \label{localization_results_sdxl}
  \resizebox{\textwidth}{!}{%
    \begin{tabular}{m{2.5cm}cccccccc}
      \toprule
      \multirow{2}{*}{\textbf{Method}} & \multicolumn{2}{c}{\textbf{One-token}} & \multicolumn{2}{c}{\textbf{Multi-token}} & \multicolumn{2}{c}{\textbf{Special-character}} & \multicolumn{2}{c}{\textbf{Sentence}} \\
      \cmidrule(lr){2-3} \cmidrule(lr){4-5} \cmidrule(lr){6-7} \cmidrule(lr){8-9}
       & ETR $\uparrow$ & AUROC $\uparrow$ & ETR $\uparrow$ & AUROC $\uparrow$ & ETR $\uparrow$ & AUROC $\uparrow$ & ETR $\uparrow$ & AUROC $\uparrow$ \\
      \midrule
      T2IShield & 90.0 & 90.4 & 80.6 & 80.9 & 83.2 & 82.6 & 71.8 & 72.2 \\
      NaviT2I  & \underline{97.1} & \underline{97.6} & \underline{96.5} & \underline{96.5} & \underline{90.2} & \underline{88.1} & \underline{80.7} & \underline{79.2} \\
      \textbf{\textit{NDDL} (Ours)} & \textbf{98.2} & \textbf{98.1} & \textbf{97.2} & \textbf{96.6} & \textbf{94.5} & \textbf{96.0} & \textbf{84.3} & \textbf{84.6} \\
      \bottomrule
    \end{tabular}%
  }
\end{table}

\subsection{Evaluation results on Pixart-$\alpha$}\label{pixart_results_appendix}

Table~\ref{defense_results_pixart} presents a comprehensive comparison of various defense methods on Pixart-$\alpha$. When evaluated against the attack on Pixart-$\alpha$, \textit{NDDL} achieves the best overall performance in both backdoor detection and trigger localization, outperforming all competing defenses. These results further substantiate the generalizability of \textit{NDDL} to the DiT architecture.
\begin{table}[htbp]
  \centering
  \caption{Evaluation results of defense methods on Pixart-$\alpha$.}
  \label{defense_results_pixart}
  \begin{tabular}{lcccc}
    \toprule
    \multirow{2}{*}{\textbf{Method}} & \multicolumn{2}{c}{\textbf{Detection}} & \multicolumn{2}{c}{\textbf{Localization}} \\
    \cmidrule(lr){2-3} \cmidrule(lr){4-5}
     & ACC $\uparrow$ & AUROC $\uparrow$ & ETR $\uparrow$ & AUROC $\uparrow$ \\
    \midrule
    UFID      & 52.5 & 50.6 & -- & -- \\
    NaviT2I   & \underline{85.5} & \underline{87.0} & \underline{78.2} & \underline{78.2} \\
    \textbf{\textit{NDDL} (Ours)} & \textbf{92.3} & \textbf{91.7} & \textbf{87.1} & \textbf{87.5} \\
    \bottomrule
  \end{tabular}
\end{table}

\end{document}

%% file: math_commands.tex
\usepackage{amsmath,amsfonts,bm}

\def\eqref#1{equation~\ref{#1}}

\def\1{\bm{1}}

\DeclareMathAlphabet{\mathsfit}{\encodingdefault}{\sfdefault}{m}{sl}
\SetMathAlphabet{\mathsfit}{bold}{\encodingdefault}{\sfdefault}{bx}{n}

